\documentclass[sn-mathphys-ay]{sn-jnl}

\usepackage{graphicx}%
\usepackage{multirow}%
\usepackage{amsmath,amssymb,amsfonts}%
\usepackage{amsthm}%
\usepackage{mathrsfs}%
\usepackage[title]{appendix}%
\usepackage{xcolor}%
\usepackage{textcomp}%
\usepackage{manyfoot}%
\usepackage{booktabs}%
\usepackage{algorithm}%
\usepackage{algorithmicx}%
\usepackage{algpseudocode}%
\usepackage{listings}%
\usepackage{mathtools}
\usepackage{longtable}
\usepackage{graphicx}

\theoremstyle{thmstyleone}%
\newtheorem{theorem}{Theorem}
\newtheorem{lemma}{Lemma}%
\theoremstyle{thmstyletwo}%
\newtheorem{remark}{Remark}%
\theoremstyle{thmstylethree}%
\newtheorem{definition}{Definition}%

\begin{document}
\title[Article Title]{Manifold Filter-Combine Networks}

\author[1]{\fnm{David R.} \sur{Johnson}}\email{davejohnson408@u.boisestate.edu}
\author[2]{\fnm{Joyce A.} \sur{Chew}}\email{joycechew@math.ucla.edu}
\author[3]{\fnm{Siddharth} \sur{Viswanath}}\email{siddharth.viswanath@yale.edu}
\author[4]{\fnm{Edward} \sur{De Brouwer}}\email{edward.debrouwer@gmail.com}
\author[2]{\fnm{Deanna} \sur{Needell}}\email{deanna@math.ucla.edu}
\author[3,5]{\fnm{Smita } \sur{Krishnaswamy}}\email{smita.krishnaswamy@yale.edu}
\author*[1,6]{\fnm{Michael} \sur{Perlmutter}}\email{mperlmutter@boisestate.edu}

\affil[1]{\orgdiv{Program in Computing}, \orgname{Boise State University}, \orgaddress{
\city{Boise},  \state{ID}, \country{USA}}, \postcode{83725}}

\affil[2]{\orgdiv{Department of Mathematics}, \orgname{UCLA}, \orgaddress{ \city{Los Angeles},  \state{CA}, \country{USA},\postcode{90095}}}

\affil[3]{\orgdiv{Department of Computer Science}, \orgname{Yale University}, \orgaddress{ \city{New Haven},  \state{CT}, \country{USA}, \postcode{06511}}}

\affil[4]{\orgdiv{Genentech}, \orgaddress{ \city{San Francisco},  \state{CA}, \country{USA}, \postcode{94080}}}

\affil[5]{\orgdiv{Department of Genetics}, \orgname{Yale University}, \orgaddress{ \city{New Haven},  \state{CT}, \country{USA}, \postcode{06511}}}

\affil[6]{\orgdiv{Department of Mathematics}, \orgname{Boise State University}, \orgaddress{
\city{Boise},  \state{ID}, \country{USA}}, \postcode{83725}}

\abstract{
In order to better understand manifold neural networks (MNNs), we introduce Manifold
Filter-Combine Networks (MFCNs).
Our filter-combine framework parallels the popular
aggregate-combine paradigm for graph neural networks (GNNs) and naturally suggests
many interesting families of MNNs which can be interpreted as manifold analogues of
various popular GNNs. We propose a method for implementing MFCNs on high-dimensional point clouds that relies on approximating an underlying manifold by a sparse graph. We then
prove that our method is consistent in the sense that it converges to a continuum limit as
the number of data points tends to infinity, and we numerically demonstrate its effectiveness on real-world and synthetic data sets.}


\maketitle

\section{Introduction}\label{sec: introduction}

The field of geometric deep learning \citep{bronstein2017geometric,bronstein2021geometric} aims to extend the success of deep learning from data such as images, with a regular grid-like structure, to more irregular domains such as graphs and manifolds. Most notably, graph neural networks (GNNs) \citep{kipf:classGCNN2017,xu2018how,hamilton2017inductive,velivckovic2017graph,Defferrard2018,levie2019transferability} have rapidly emerged as an extremely active area of research  \citep{ZHOU202057,wu2020comprehensive} and are utilized in industrial applications such as Google Maps \citep{derrow2021eta} and Amazon's product recommender system \citep{Wang2022}. 

By contrast, the manifold side of geometric deep learning is much less explored. Moreover, much of the existing literature is either limited to specific manifolds \citep{cohen:sphericalCNNs2018, esteves2020spin} or to two-dimensional surfaces embedded in three-dimensional space \citep{masci2015geodesic,boscaini2015learning,Masci:geoCNN2015, schonsheck2022parallel}  and cannot be applied to high-dimensional manifolds (possibly with co-dimension greater than one). 
 This is despite the fact that unsupervised manifold learning algorithms \citep{tenenbaum:isomap2000,coifman:diffusionMaps2006,maaten:tSNE2008} are commonly used for representing higher-dimensional data \citep{van2018recovering, kuchroo2021topological, moyle2021structural}). These algorithms consider high-dimensional point clouds $\{x_i\}_{i=1}^n\subseteq \mathbb{R}^D$ and assume that the points satisfy the manifold hypothesis, i.e., that they lie upon some low-dimensional manifold $\mathcal{M}$. They then aim to find a low-dimensional representation of the $x_i$ which captures the intrinsic data geometry.

 Inspired by the successes of graph neural networks and manifold learning, several recent works have proposed  manifold neural networks (MNNs) \citep{wang2021stability,wang2021stabilityrel} and manifold scattering transforms \citep{perlmutter:geoScatCompactManifold2020,mcewen2021scattering,chua2024generalizing}. These works define convolution in terms of a manifold Laplacian $\mathcal{L}$ such as the (negative) Laplace Beltrami operator $-\Delta=\text{div}\circ\nabla$, paralleling common spectral GNNs \citep{bruna:spectralNN2014,Defferrard2018,Levie:CayleyNets2017} that define convolution in terms of the graph Laplacian. 
 In particular, several recent papers \citep{chew2023convergence,chew2022geometric,wangsparseconvergence}  have introduced numerical methods for implementing MNNs on point clouds satisfying the manifold hypothesis. These methods approximate the manifold $\mathcal{M}$ by a graph $G_n$ whose vertices are the data points $x_i$. Then, the eigendecomposition of the associated graph Laplacian $\mathbf{L}_n$ is used to approximate the eigenfunctions and eigenvalues of the manifold Laplacian.  The authors establish the accuracy and statistical consistency of these methods by proving that their discrete implementations converge to continuum limits as the number of data points tends to infinity under various assumptions.

In this work, in order to to enrich theoretical understanding of MNNs and establish a flexible framework for designing families of MNNs, we introduce \emph{Manifold Filter-Combine Networks}. Our manifold filter-combine paradigm deliberately parallels the aggregate-combine  framework introduced in \citet{xu2018how} to understand GNNs. It naturally leads one to consider many interesting classes of MNNs which can be understood as the manifold counterparts of various popular GNNs. We establish sufficient conditions for such networks to converge to a continuum limit as the number of sample points $n$ tends to infinity. Finally, we conduct numerical experiments which show that various MFCN models are effective on real-world and synthetic data sets.

\subsection{Contributions}\label{sec: contributions}

The primary contributions of this work are:

\begin{enumerate}
    \item We introduce \emph{Manifold Filter-Combine Networks} as a novel framework for understanding and designing manifold neural networks.
    \item In Theorem \ref{thm: Filter Error short}, we provide sufficient conditions for our numerical implementation of spectral filters on finite point clouds to converge to a continuum limit as the number of sample points $n\rightarrow\infty$. Notably, unlike most of the existing literature on the convergence of spectral filters and MNNs, we do not assume that the points are generated from the uniform distribution. 
    \item In Theorem \ref{thm: bound given filter bound short}, we build upon Theorem \ref{thm: Filter Error short} by providing sufficient conditions for our numerical implementation of MFCNs to converge to a continuum limit as the number of sample points $n\rightarrow\infty$. Unlike previous results on the convergence of MNNs, our rate of convergence does not exhibit exponential dependence on the depth of the network. Instead, we show that we may achieve linear dependence on network depth if the weights of the network are properly normalized.
    
    \item We conduct numerical experiments which empirically validate our theory and show that various MFCN models  
    are effective on real-world and synthetic data sets. As part of these experiments, we also introduce Infogain, a novel method for selecting the scales utilized in diffusion wavelets.
\end{enumerate}

\subsection{Notation}

We let $\mathcal{M}$ be a compact, connected, $d$-dimensional Riemannian manifold and let $\mu$ be a probability distribution on $\mathcal{M}$ with density $\rho$.  We let $L^2(\mathcal{M})$ denote the set of functions that are square integrable with respect to $\mu$ and $\mathcal{C}(\mathcal{M})$ denote the set of continuous functions on $\mathcal{M}$. We let $\mathcal{L}$ denote a differential operator which is assumed to be symmetric on $L^2(\mathcal{M})$ and is to be thought of as a manifold Laplacian.  We let $\{\phi_i\}_{i=1}^\infty$ denote an orthonormal basis of eigenfunctions $\mathcal{L}\phi_i=\lambda_i\phi_i.$  We will use these eigenfunctions to define Fourier coefficients denoted by $\widehat{f}(i).$

In much of our analysis, we will assume that $\mathcal{M}$ is unknown and that we only have access to a function $f\in\mathcal{C}(\mathcal{M})$ evaluated at sample points $\{x_j\}_{j=1}^n\subseteq\mathbb{R}^D$. In this setting, we will let $P_n:\mathcal{C}(\mathcal{M})\rightarrow\mathbb{R}^n$ be the normalized evaluation operator  
$(P_nf)(j)=\frac{1}{\sqrt{n}}f(x_j),$ 
and let $G_n$ denote a graph whose vertices are the sample points $x_j$. We will let $\mathbf{L}_n$ denote the graph Laplacian associated to $G_n$ and let $\phi_i^n$ be an orthonormal basis of eigenvectors, $\mathbf{L}_n\phi_i^n=\lambda^n_i\phi_i^n,$ $0=\lambda^n_1\leq \lambda^n_2\leq \ldots\leq \lambda^n_n$. Analogous to the continuous setting, we will use the $\phi_i^n$ to define discrete Fourier coefficients $\widehat{\mathbf{x}}(i)$.

The primary purpose of this paper is to introduce and analyze a family of neural networks to process functions defined on $\mathcal{M}$. Towards this end, we will let $F=(f_1,\ldots,f_C)$ denote a row-vector valued function and let $F^{(\ell)}$ denote the hidden representation in the $\ell$-th layer of our network, with $F^{(0)}=F$. When we approximate our network on $G_n$, we will instead assume that we are given an $n\times C$ data matrix $\mathbf{X}=(\mathbf{x}_1,\ldots,\mathbf{x}_C$).

\subsection{Organization}
The rest of this paper is organized as follows. In Section \ref{sec: background}, we will review necessary background material from signal processing and neural networks on graph and manifolds and will also provide a brief overview of manifold learning. In Section \ref{sec: MFCN}, we will introduce a novel class of manifold neural networks which we name Manifold Filter-Combine Networks (MFCNs), paralleling commonly used aggregate-combine networks for graph structured data, and also provide a numerical method for implementing MFCNs on point clouds satisfying the manifold hypothesis. In Section \ref{sec: convergence}, we will then provide convergence guarantees for this method as the number of samples points tends to infinity. We will conduct numerical experiments in Section \ref{sec: results}, before providing a brief conclusion in Section \ref{sec: conclusion}. We will provide proofs of all of our theoretical results as well as further experimental details in the appendix.

\section{Background and Preliminaries}\label{sec: background}
\subsection{Graph Signal Processing and Convolutions}\label{sec: background gsp}
Graph signal processing \citep{shuman:emerging2013,Ortega2018} aims to extend traditional signal processing methods such as the Fourier transform or the wavelet transform to signals defined on the vertices of a graph by using the eigendecomposition of the graph Laplacian. Formally, we consider a weighted graph $G=(V,E,\eta)$, with vertex set $V=\{v_1,\ldots,v_n\}$, edge set $E$,  and weighting function $\eta:E\rightarrow [0,\infty)$. We then define the (unnormalized) graph Laplacian by 
$
\mathbf{L}=\mathbf{D}-\mathbf{A},
$ 
where $\mathbf{A}\in\mathbb{R}^{n\times n}$ is the weighted adjacency matrix of $G$ and $\mathbf{D}$ is the corresponding diagonal degree matrix, i.e., $\mathbf{D}_{i,i}=\sum_{j=1}^n\mathbf{A}_{i,j}$ and $\mathbf{D}_{i,j}=0$ for all $i\neq j$. (Alternatively, one could also consider the normalized graph Laplacian $\mathbf{L}_{\text{normalized}}=\mathbf{D}^{-1/2}\mathbf{L}\mathbf{D}^{-1/2}$.) 
It is known (see, e.g., \citet{shuman:emerging2013}) that $\mathbf{L}$ is a positive definite matrix and admits an orthonormal basis of eigenvectors $\mathbf{u}_1,\ldots,\mathbf{u}_n$ with $\mathbf{L}\mathbf{u}_i=\lambda_i\mathbf{u}_i$, $0=\lambda_1\leq\lambda_2\leq\ldots\leq\lambda_n$. 

Given a signal (function) $\mathbf{x}$ defined on the vertices $V$, we will, in a minor abuse of notation, identify $\mathbf{x}$ with the vector $\mathbf{x}\in\mathbb{R}^n,$ $\mathbf{x}(i)=\mathbf{x}(v_i)$. We then define the graph Fourier transform by 
$
\widehat{\mathbf{x}}(i)=\langle\mathbf{u}_i,\mathbf{x}\rangle_2.
$
Since the $\mathbf{u}_i$ form an orthonormal basis, we obtain the Fourier inversion formula
$$
\mathbf{x}=
\sum_{i=1}^n \widehat{\mathbf{x}}(i)\mathbf{u}_i,
$$
where for $1\leq i \leq n$, $\widehat{\mathbf{x}}(i)$ is the generalized Fourier coefficient defined by $\widehat{\mathbf{x}}(i)=\langle\mathbf{u}_i,\mathbf{x}\rangle_2.$

We note that in the case that $G=C_n$ is a cycle graph, then the graph Fourier transform reduces to the classical discrete Fourier transform (up to multiplication by a constant). 
Given the graph Fourier transform, one may then define convolution as multiplication in the Fourier domain, analogous to the convolution theorem in classical signal processing, 
i.e., 
$$
\mathbf{x}\star\mathbf{y}=\sum_{i=1}^n\widehat{\mathbf{x}}(i)\widehat{\mathbf{y}}(i)\mathbf{u}_i,
$$
so that $\widehat{\mathbf{x}\star\mathbf{y}}(i)=\widehat{\mathbf{x}}(i)\widehat{\mathbf{y}}(i)$. 
This notion of convolution is used in graph signal processing to construct filters of the form $\mathbf{x}\mapsto\mathbf{y}\star\mathbf{x}$ where each signal of interest $\mathbf{x}$ is convolved with a filter $\mathbf{y}$.
Often, it is useful to assume that the frequency response of the filter is a function of the eigenvalues, i.e., that there exists a function $w$ such that $\widehat{\mathbf{y}}(i)=w(\lambda_i)$. This leads to spectral filters of the form
\begin{equation}\label{eqn: simple graph convolutional filter}
w(\mathbf{L})\mathbf{x}=\sum_{i=1}^nw(\lambda_i)\widehat{\mathbf{x}}(i)\mathbf{u}_i.
\end{equation}

\subsection{Manifold Learning}\label{sec: background manifold learning}

The \emph{manifold hypothesis} is the belief that many high-dimensional data sets have intrinsic low-dimensional structure. Formally, one is given a high-dimensional data set $\{x_i\}_{i=1}^n\subseteq \mathbb{R}^D$ and assumes that the data points $x_i$ lie upon (or near) some unknown $d$-dimensional Riemannian manifold, $\mathcal{M}$, where $d\ll D$. 

Manifold learning \citep{van2009dimensionality, izenman2012introduction,lin2015geometric,Moon2018Manifold} refers to a collection of algorithms designed for data sets that are believed to satisfy the manifold hypothesis. 
Given the data points $x_i$, one aims to produce a low-dimensional  representation that approximates the intrinsic geometry of~$\mathcal{M}$. Intuitively, such methods can be thought of as achieving dimensionality reduction by detecting  non-linear patterns in the data, analogous to the manner in which Principal Component Analysis (PCA) is used to detect linear patterns. 

Many popular manifold learning algorithms such as diffusion maps \citep{coifman:diffusionMaps2006} and Laplacian eigenmaps \citep{belkin:laplacianEigen2003} operate by constructing a weighted graph $G_n=(V_n,E_n, \eta)$ where the vertices are the data points, i.e., $V_n=\{x_i\}_{i=1}^n$. The edge sets can be constructed in a variety of ways. Here we  focus on \begin{enumerate}
\item $\epsilon$-graphs, in 
which one places an edge between $x_i$ and $x_j$ if $\|x_i-x_j\|_2\leq \epsilon$, and 
\item $k$-NN graphs, in which each $x_i$ is connected to its $k$ nearest neighbors.
\end{enumerate}
In either case, if desired, one can then assign edge weights using the weighting function $\eta$. Typically, one assumes that the edge weights are a non-increasing function of the Euclidean distance between $x_i$ and $x_j$ (in the ambient space $\mathbb{R}^D$).

Laplacian eigenmaps \citep{belkin:laplacianEigen2003} consider the eigendecomposition\footnote{In \citep{belkin:laplacianEigen2003}, the authors actually consider a \emph{generalized} eigenvalue problem. We ignore this for ease of exposition.} of the graph Laplacian $\mathbf{L}_n=\mathbf{D}_n-\mathbf{A}_n$ associated to $G_n$, $\mathbf{L}_n\phi_i^n=\lambda_i^n\phi_i^n$, where  $0=\lambda_1^n\leq \lambda^n_2\leq\ldots \leq\lambda^n_n$ for $1\leq i\leq n$. Dimensionality reduction is then achieved via the eigenmap that maps each $x_j\in\mathbb{R}^D$ to a point in $\mathbb{R}^m$ defined in terms of the $j$-th coordinate of the first $m$ non-trivial eigenvectors: $$x_j\mapsto(\phi_2^n(j),\ldots,\phi^n_{m+1}(j))\in\mathbb{R}^m.$$ (The first eigenvector is omitted because $\lambda_1$ is always zero and $\phi_1^n$ is a constant vector.)

Diffusion maps \citep{coifman:diffusionMaps2006} build off of Laplacian eigenmaps by considering a Markov normalized diffusion operator $\mathbf{P}_n$ that describes the transitions of a random walker over the vertices (after a suitable renormalization of the edge weights) and can also be thought of as a discrete approximation of an underlying heat kernel. The diffusion map is then defined by 
$$x_j\mapsto(\omega_2^t\psi_2^n(j),\ldots,\omega_{m+1}^t\psi^n_{m+1}(j))\in\mathbb{R}^m,$$
where $\mathbf{P}_n\psi_i^n=\omega_i\psi_i^n$, and $t$ is a parameter referred to as diffusion time.

In order to justify the intuition that diffusion maps, Laplacian eigenmaps,  and other related algorithms capture the intrinsic geometry of the data, one may aim to prove that the graph Laplacian $\mathbf{L}_n$ is consistent in the sense that it converges to a limiting operator defined on the underlying manifold $\mathcal{M}$ as the number of data points tends to infinity. Typically, one views $\mathbf{L}_n$ as discrete approximation of a differential operator $\mathcal{L}$ such as the (negative) Laplace-Beltrami Operator, $-\Delta=-\text{div}\circ\nabla$, where $\nabla$ is the intrinsic gradient and $\text{div}$ is the corresponding divergence operator.
One then attempts to show   that the eigendecomposition of $\mathbf{L}_n$ approximates the eigendecomposition of $\mathcal{L}$, $\mathcal{L}\phi_i=\lambda_i\phi_i$, in the sense that 
\begin{equation}\label{eqn: spectral convergence}
\lim_{n\rightarrow\infty} \lambda_i^{n} = \lambda_i,\quad \text{and}\quad 
\lim_{n\rightarrow\infty} \|\phi_i^{n}-P_n\phi_i\|_2=0,
\end{equation}
 where  $P_n$ is the normalized evaluation operator defined by 
\begin{equation}\label{eqn: Pn}
P_nf(i)=\frac{1}{\sqrt{n}}f(x_i),
\end{equation}
for continuous functions $f$. 

Results along the lines of \eqref{eqn: spectral convergence}, often with quantitative rates of convergence, as well as pointwise results such as $$\lim_{n\rightarrow\infty} \mathbf{L}_nP_nf(i)=\mathcal{L}f(x_i),$$  have been established in numerous works such as \citet{von2008consistency, trillos2018variational, dunson2021spectral, cheng2022eigen, Calder2019,belkin2008towards,gine2006empirical, dunlop2020large, calder2019consistency}.

\subsection{Signal Processing and Spectral Convolution on Manifolds}\label{sec: background manifold convolution}

In this section, we will discuss the extension of convolutional operators to manifolds. While many possible solutions to this problem have been proposed including methods based on parallel transport \citep{schonsheck2022parallel}, local patches \citep{masci2015geodesic,boscaini2015learning,Masci:geoCNN2015}, or Fr{\'e}chet means \citep{chakraborty2018manifoldnet}, here we will focus on spectral convolutional methods that are the manifold analogue of the graph signal processing methods discussed in Section \ref{sec: background gsp}.

 Similarly to Section \ref{sec: background manifold learning}, we consider a  compact, connected, $d$-dimensional Riemannian manifold without boundary $\mathcal{M}$ embedded in $\mathbb{R}^D$, $D\gg d$. 
We will let $\mu$ be a probability distribution on $\mathcal{M}$ with a density  $\rho: \mathcal{M}\rightarrow[0,\infty)$ (with respect to the Riemannian volume form). As in \citet{Calder2019},  we will assume throughout that there exist constants $\rho_{\text{min}}$ and $\rho_{\text{max}}$ such that $
0<\rho_{\text{min}}\leq \rho \leq \rho_{\text{max}}<\infty
$
and that $\rho$ is at least $\mathcal{C}^{2,\nu},$ i.e., that its second derivatives are at least $\nu$-H\"older continuous.
We  
define an $L^2$ norm and inner product by  $$
\|f\|_{L^2(\mathcal{M})}^2=\int_\mathcal{M}|f|^2d\mu<\infty,\quad
\text{and}\quad 
\langle f,g\rangle_{L^2(\mathcal{M})}=\int_\mathcal{M} fgd\mu,
$$
and let $L^2(\mathcal{M})$ denote the set of functions such that  $\|f\|_{L^2(\mathcal{M})}$ is finite.
 (Other $L^p(\mathcal{M})$ norms are defined similarly, i.e., $
\|f\|_{L^p(\mathcal{M})}^p=\int_\mathcal{M}|f|^pd\mu,
$ for $1\leq p<\infty$ and $\|f\|_{L^\infty(\mathcal{M})}$ is defined as the essential supremum.)

 We let $\mathcal{L}$ denote a differential operator, which is to be thought of as a (weighted) manifold Laplacian, and assume that 
$\mathcal{L}$ is symmetric in the sense that $\langle \mathcal{L} f,g\rangle_{L^2(\mathcal{M})}=\langle  f,\mathcal{L} g\rangle_{L^2(\mathcal{M})}$ for all  smooth functions $f$ and $g$. We will assume that $\mathcal{L}$ has an orthonormal (with respect to $\langle\cdot,\cdot\rangle_{L^2(\mathcal{M})}$) basis of eigenfunctions $\mathcal{L}\phi_i=\lambda_i\phi_i$,  $\lambda_1\leq \lambda_2\leq\ldots$, $1\leq i <\infty$. Crucially, we note that if $\rho$ is a non-uniform (i.e., non-constant) density, then we cannot chose $\mathcal{L}$ to be the Laplace Beltrami operator $-\Delta=-\text{div}\circ\nabla$ since $-\Delta$ is not symmetric with respect to the inner product $\langle f, g\rangle_{L^2(\mathcal{M})}=\int_\mathcal{M} fgd\mu$. Instead we will consider operators such as $\mathcal{L}f=-\frac{1}{\rho}\text{div}(\rho^2 \nabla f)$, which can be verified to be symmetric on $L^2(\mathcal{M})$. 

Analogous to the graph signal processing methods discussed in Section \ref{sec: background gsp}, we may now  
  define a generalized Fourier series by setting
$
\widehat{f}(i)=\langle f,\phi_i \rangle_{L^2(\mathcal{M})},
$
and obtain the Fourier inversion formula 
$$
f(x)=\sum_{i=1}^\infty \langle \phi_i , f\rangle_{L^2(\mathcal{M})} \phi_i(x)=\sum_{i=1}^\infty \widehat{f}(i)\phi_i(x).
$$
Notably, in the case where $\mathcal{M}=S^1$ is the unit circle embedded in two-dimensional space, and $\rho$ is the uniform density, this generalized Fourier series reduces to the classical Fourier series widely used in signal processing for periodic functions defined on the real line (which can be interpreted as functions defined on the circle). 

Given the generalized Fourier series, one may then define convolution as multiplication in the Fourier domain, i.e.,
$$
(f\star g)(x)=\sum_{i=1}^\infty \widehat{f}(i)\widehat{g}(i)\phi_i(x).
$$
Additionally, given a function $w:[0,\infty)\rightarrow\mathbb{R}$, with $\|w\|_\infty<\infty$, we may define a spectral convolution operator, $w(\mathcal{L}):L^2(\mathcal{M})\rightarrow L^2(\mathcal{M})$ by 
\begin{equation}\label{eqn: specconvcontinuum} w(\mathcal{L})f=\sum_{i=1}^\infty w(\lambda_i) \widehat{f}(i) \phi_i.
\end{equation}

We note that in the case where  $\sum_{i=1}^\infty |w(\lambda_i)|^2<\infty$, applying the operator $w(\mathcal{L})$ equivalent to convolving $f$ with the function $g$ such that $\widehat{g}(i)=w(\lambda_i)$. Furthermore, we note, as was observed in Remark 1 of \citet{chew2022geometric}, that $w(\mathcal{L})$ is independent of the choice of eigenbasis. To see this, we let $\Lambda$ denote the set of distinct eigenvalues of $\mathcal{L}$ and, for $\lambda\in \Lambda$, we let $\pi_\lambda$ denote the operator which projects each function onto the corresponding eigenspace $E_\lambda$. We may then write
$$
w(\mathcal{L})f=\sum_{i=1}^\infty w(\lambda_i) \widehat{f}(i) \phi_i=\sum_{\lambda\in\Lambda}w(\lambda)\left(\sum_{\lambda_i=\lambda} \widehat{f}(i) \phi_i\right)=\sum_{\lambda\in\Lambda}w(\lambda)\pi_\lambda(f).
$$
We note that for a majority of this paper, we will use the letter ``$w$" to denote a spectral filter $w(\mathcal{L})$. This is inspired by the wavelets used in several of the MFCN models we consider in Section \ref{sec: results}. However, we may use other letters depending on context, and, in particular, we will write $a(\mathcal{L})$ if we want to emphasize that the spectral filter is a low-pass filter (i.e., an averaging operator).

\subsection{Neural Networks on Graphs and Manifolds}\label{sec: GNNs and MNNs}

Using the notion of spectral convolution discussed in Section \ref{sec: background gsp}, one can define a simple spectral GNN of the form
$$\mathbf{x}^{(\ell+1)}=\sigma\left(w^{(\ell)}(\mathbf{L})\mathbf{x}^{(\ell)}\right)$$
where $\sigma$ is a pointwise, nonlinear activation function and $w^{(\ell)}(\mathbf{L})$ is defined as in \eqref{eqn: simple graph convolutional filter}.
If one is given multiple initial graph signals $\mathbf{x}_1,\ldots, \mathbf{x}_C$ organized into a data matrix $\mathbf{X}=(\mathbf{x}_1,\ldots,\mathbf{x}_C)$ and uses multiple filters in each layer, then the layerwise update rule can be extended to 
\begin{equation}\label{eqn: CNN style GNN}
\mathbf{x}^{(\ell+1)}_k=\sigma\left(\sum_{j=1}^C w^{(\ell)}_{j,k}(\mathbf{L})\mathbf{x}^{(\ell)}_k\right).
\end{equation}
Similarly, in the manifold case, if one is given $C$ input functions $f_1,\ldots,f_C:\mathcal{M}\rightarrow \mathbb{R}$, then a spectral MNN can be defined with a layerwise update rule of 
\begin{equation}\label{eqn: simple MNN}
f^{(\ell+1)}_k=\sigma\left(\sum_{j=1}^C w^{(\ell)}_{j,k}(\mathcal{L})f^{(\ell)}_k\right),
\end{equation}
where $\mathcal{L}$ is a manifold Laplacian such as the Laplace-Beltrami operator.  
In either \eqref{eqn: CNN style GNN} or \eqref{eqn: simple MNN}, if each filter $w^{(\ell)}_{j,k}$ belongs to a parameterized family of functions such as Chebyshev polynomials, one can then attempt to learn the optimal parameters from training data.

However, many popular graph neural networks take a different approach. Rather than using multiple learnable filters for each input channel and then summing across channels, they instead filter each graph signal with a pre-designed operator (or operators) and then learn relationships between the filtered input signals. For example, the Graph Convolutional Network (GCN)\footnote{Here, we use the term GCN to refer to the specific network introduced in \citet{kipf2016semi}. We will use the term GNN to refer to a general graph neural network} \citep{kipf:classGCNN2017} performs a predesigned aggregation 
$
\mathbf{X}\rightarrow \widehat{\mathbf{A}}\mathbf{X}
$, where $\widehat{\mathbf{A}}=(\mathbf{D}+\mathbf{I})^{-1/2}(\mathbf{A}+\mathbf{I})(\mathbf{D}+\mathbf{I})^{-1/2}$ and utilizes a right-multiplication by a trainable weight matrix $\Theta$ to learn relationships between the channels\footnote{The matrix $\widehat{\mathbf{A}}$ can be obtained by applying the polynomial $h(\lambda)=1-\lambda/2$ to a normalized version of the graph Laplacian and then some adjustments which help with the training of the network. Therefore, we can essentially think of the operation $\mathbf{x}\rightarrow\widehat{\mathbf{A}}\mathbf{x}$ as a spectral convolution.}. This leads to the layerwise update rule
$
\mathbf{X}^{(\ell+1)}=\sigma(\widehat{\mathbf{A}}\mathbf{X}^{(\ell)}\Theta^{(\ell)}).$

GCNs, and other similar networks, can be thought of as \emph{Aggregate-Combine} networks \citep{xu2018how} since each layer consists of two linear transformations. First, the aggregation operator performs an operation similar to $\mathbf{X}^{(\ell)}\rightarrow\widehat{\mathbf{A}}\mathbf{X}^{(\ell)}$ which aggregates information over each graph neighborhood. This operator effectively performs local averaging and makes the hidden representation of each vertex more similar to its immediate neighbors.  Second, the combine operation, e.g., $\widehat{\mathbf{A}}\mathbf{X}^{(\ell)}\rightarrow\widehat{\mathbf{A}}\mathbf{X}^{(\ell)}\Theta^{(\ell)}$ then learns novel combinations of the input channels.  In the following section, we will introduce a class of manifold neural networks which extends the aggregate-combine framework to the manifold setting. However, in our analogue of the aggregate step, we will instead utilize more general operators defined in the spectral domain. (From the spectral perspective, the aggregation operators utilized in GCNs can be seen as low-pass filters \citep{nt2019revisiting}.) In this manner, our method will be utilizing ideas from both spectral networks as well as from aggregate-combine networks.

\section{Manifold Filter Combine Networks}\label{sec: MFCN}

We now introduce a novel framework,  which we refer to as the \emph{filter-combine} paradigm, for thinking about manifold neural networks constructed via spectral convolutional operators as defined as in \eqref{eqn: specconvcontinuum}.
 This terminology parallels the aggregate-combine framework introduced in \citet{xu2018how} to understand GNNs. (See Section \ref{sec: GNNs and MNNs}.) 
 Notably, we use the term ``filter'' rather than ``aggregate'' because our filters  are not required to be localized averaging operations such as those used in common message-passing GNNs.  

We will assume that our input data is a row-vector\footnote{We define the output of $F$ to be $\mathbb{R}^{1\times C}$ in order to highlight the parallels with the data matrices commonly considered in the GNN literature where rows correspond to vertices and columns correspond to features.} valued function $F\in L^2(\mathcal{M},\mathbb{R}^{1\times C})$, $F=(f_1,\ldots,f_{C})$, where each $f_i\in L^2(\mathcal{M})$ can is interpretted as a feature (or an input channel).

Each hidden layer of the network will consist of the following five steps: 
(i) We \textit{Filter} each input feature $f_k$ by a family of spectral operators $w_{j,k}(\mathcal{L})$, $1\leq j\leq J$. (ii) For each fixed $j$, we  \textit{Combine Features} so that the filtered feature functions $\tilde{f}_{j,k}=w_{j,k}(\mathcal{L})f_k$ get combined into $C'$ new feature functions $g_{j,k}$, where each $g_{j,k}$, $1\leq k \leq C'$, is a linear combination of the $\tilde{f}_{j,k}$.   (iii) Then, for each fixed $k$, we \textit{Combine Filters}, i.e., we map the $\{ g_{j,k}\}_{j=1}^J$ to $\{\tilde{g}_{j,k}\}_{j=1}^{J'}$ where each $\tilde{g}_{j,k}$, $1\leq j\leq J'$, is a linear combination of the $ g_{j,k}$.  (iv) We apply some non-linear, nonexpansive pointwise \textit{Activation} function $\sigma$ to each of the $\tilde{g}_{j,k}$,  to obtain $h_{j,k}=\sigma\circ \tilde{g}_{j,k}$. (v) Lastly, we \textit{Reshape} the $\{h_{j,k}\}_{1\leq j \leq J',1\leq k\leq C'}$ into $\{f'_i\}_{i=1}^{\tilde{C}}$, where $\tilde{C}=C'J'$. 

Explicitly, we define the layerwise update rules as follows. Given a function
$F^{(\ell)}=(f_1^{(\ell)},\ldots,f_{C_\ell}^{(\ell)})$, with $C_\ell$ input features (channels), and a filter bank, $w_{j,k}^{(\ell)}$,  $1\leq j\leq J_\ell$, of $J_\ell$ filters for each feature, we define $F^{(\ell+1)}=(f_1^{(\ell+1)},\ldots,f_{C_{\ell+1}}^{(\ell+1)})$ by:

\begin{align*}
\text{Filter:}&\quad
\tilde{f}^{(\ell)}_{j,k}=w^{(\ell)}_{j,k}(\mathcal{L})f^{(\ell)}_k,\quad 1\leq j \leq J_\ell, 1\leq k\leq C_\ell\\
\text{Combine Features:}&\quad
g_{j,k}^{(\ell)}=\sum_{i=1}^{C_{\ell}}\tilde{f}^{(\ell)}_{j,i}\theta^{(\ell,j)}_{i,k},\quad 1\leq j\leq J_\ell, 1\leq k \leq C'_\ell\\
\text{Combine Filters:}&\quad
\tilde{g}^{(\ell)}_{j,k}= \sum_{i=1}^{J_\ell} \alpha^{(\ell,k)}_{j,i}g_{i,k},\quad 1\leq j\leq J_\ell',1\leq k\leq C'_\ell\\
\text{Activation:}&\quad
h_{j,k}^{(\ell)}=\sigma^{(\ell)}\circ \tilde{g}_{j,k}^{(\ell)},\quad 1\leq j\leq J_\ell', 1\leq k \leq C'_\ell\\
\text{Reshaping:}&\quad
f^{(\ell+1)}_{(j-1)C_{\ell}+k} = h^{(\ell)}_{j,k}, \quad 1\leq j\leq J_\ell',1\leq k\leq C'_\ell.
\end{align*}
We note that this results in  $C_{\ell+1}=J'_\ell C_{\ell}'$ features to be input into the next layer and that we set  $F^{(0)}=F,C_0=C$.
Above, each of the $\alpha_{j,k}^{(\ell,k)}$ and $\theta_{i,k}^{(\ell,k)}$  are scalars, which, if desired, can be taken to be trainable parameters. A graphical depiction of this sequence of operations (using the numerical method introduced in Section \ref{sec: MFCN discrete implementation}) is given in Figure \ref{fig:architecture}. We remark that the first three steps are largely inspired by the Aggregate and Combine steps used in GNNs. In particular, the Filter step is meant to parallel the Aggregate step, which can be thought of as a low-pass filtering  of the node features \citep{nt2019revisiting}, although here, we also allow for other types of filters, such as high-pass filters or wavelets (band-pass filters). The next step, Combine Features, learns new features which are combinations of the original input features, similar to the Combine step in GNNs. The Combine-Filters step is inspired by the work of \citet{zarka2020separation} which we will further discuss in Section \ref{ex: learnable scattering}. It allows the network to learn new filters, which are combinations of those used in the first step. The Activation step is a standard feature of most neural networks and the Reshaping step is merely included so that each layer both inputs and outputs a row-vector valued function, allowing for multiple layers to be stacked upon each other.

We shall refer to networks constructed using the layers above as Manifold Filter-Combine Networks (MFCNs). 
As we  discuss below, the MFCN framework naturally lends itself to the construction of 
many different subfamilies of networks. For instance, rather than allowing different filters $w_{j,k}(\mathcal{L})$ for each channel, one could require that the $w_{j,k}(\mathcal{L})$ are the same for all input channels $f_k$ (i.e., $w_{j,k}(\mathcal{L})=w_j(\mathcal{L})$). Indeed, this will be the case for the networks considered in Sections \ref{ex: mcn}, \ref{ex: scat}  and \ref{ex: learnable scattering} below, and may in some settings facilate learning by reducing the number of trainable parameters. Furthermore, one could use filters which are \emph{learned}, analogous to ChebNets \citep{Defferrard2018} (discussed below in Section \ref{ex: chebnet}) and CayleyNets \citep{Levie:CayleyNets2017}, or one could use filters which are designed in advance and allow the network to learn via combining features and/or filters. The latter approach parallels the updates used in GCNs \citep{kipf2016semi} where each feature is effectively averaged  over local neighborhoods (i.e., smoothed) and then the network learns novel combinations of these features.  

We note one may effectively omit the Combine-Features step by setting $C_\ell'=C_\ell$ and setting each matrix $\Theta^{(\ell,j)}:=(\theta_{i,k}^{(\ell,j)})_{1\leq i,k\leq C_\ell}$ equal to the identity. 
Similarly, one may omit  the Combine-Filters step by setting $J_\ell=J_\ell'$ and setting the matrices whose entries are $(\alpha_{j,i}^{(\ell,k)})_{1\leq i,j\leq J_\ell}$ to the identity. 

Additionally, we note that because of the flexibility of our framework it is possible to write the same network as an MFCN in more than one way. 
For instance, if one omits the Combine-Filters step, uses a shared filter bank $\{w^{(\ell)}(\mathcal{L})_j\}_{1\leq j \leq J}$ (independent of $k$) and chooses the Combine-Features step to be independent of $j$ (i.e., setting $\theta_{i,k}^{(\ell,j)}=\theta_{i,k}^{(\ell)}$), then the corresponding MFCN layerwise update is
$$
f^{(\ell+1)}_{(j-1)C_\ell+k} = \sigma^{(\ell)}\left(\sum_{i=1}^{C_\ell}w^{(\ell)}(\mathcal{L})_j\theta^{(\ell)}_{i,k}f_i\right),
$$
which may also be obtained by using filters of the form 
$\widetilde{w}^{(\ell)}_{(j-1)C_\ell+k,i}(\mathcal{L})=w_{j}(\mathcal{L})\theta^{(\ell)}_{i,k}$ and using a Combine-Features step with $\tilde{\theta}_{i,k}^{(\ell,j)}=1$.

\begin{figure}[htbp]
    \centering
    \includegraphics[width=\linewidth]{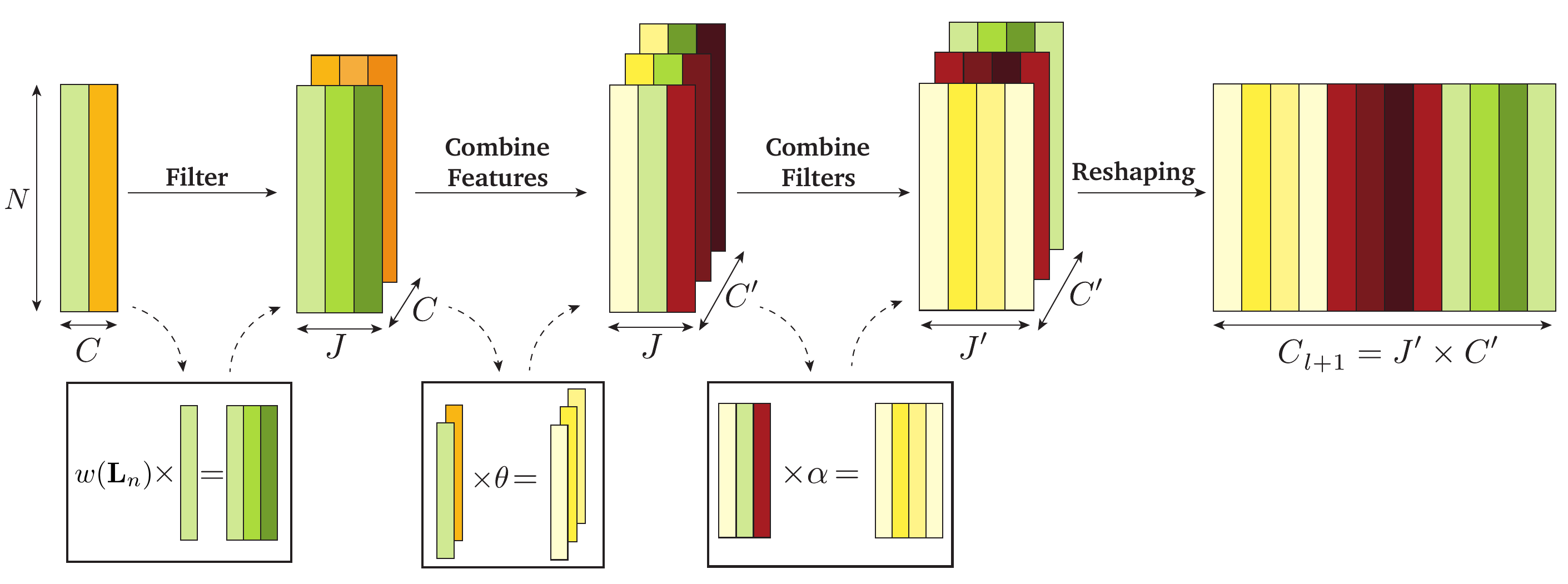}
    \caption{Illustration of Manifold Filter-Combine Networks steps. Starting from a $C$-dimensional row vector-valued function evaluated at $N$ points, the MFCN layer in turn filters, combines features, combines filters, applies a point-wise nonlinearity, and reshapes. (For conciseness, we do not visualize the activation step). }
    \label{fig:architecture}
\end{figure}

\subsection{Examples of MFCNs}\label{sec:mfcn examples}

We now discuss several examples of different subfamilies of networks which may be constructed and analyzed via the MFCN framework. These families differ in the types of filter banks utilized and in which aspects of the MFCN layer are chosen to be learnable. The flexibility of the MFCN framework enables rational design of MFCN architectures to incorporate filters, cross-feature relationships, and cross-filter relationships that are better suited for particular applications. Additionally, each part of the network may be chosen to be learnable, for improved expressivity, or predesigned, to reduce the number of trainable parameters.

\subsubsection{Manifold GCNs}\label{ex: mcn}

We first construct a simple MFCN paralleling the popular GCN  introduced in \citet{kipf2016semi}. Here, we will use a single filter $\widetilde{\mathcal{A}}=a(\mathcal{L})$, which means that $J_\ell=1$ for all layers. Furthermore, we choose the associated function  $a(\lambda)$ to be decreasing, so that $\widetilde{\mathcal{A}}$ may be thought of as a low-pass filter.
 Since there is only one filter, we omit the Combine-Filters step, and  and let the matrix $\Theta^{(\ell)} = (\theta^{(\ell,1)}_{i,k})_{1\leq i\leq C_\ell,1\leq k \leq C'_\ell}$ be a learnable weight matrix. Then our layerwise update rule becomes $f_k^{(\ell+1)}=\sigma \left ( \sum_{i=1}^{C_\ell}\theta_{i,k}^{(\ell,1)}\widetilde{A}f_k\right )$, which may be written compactly as 
$F^{(\ell+1)}=\sigma\left(\widetilde{A}F^{(\ell)}\Theta^{(\ell)}\right)$.
As is the case with the GCN, this network is designed to apply smoothing operations with the low-pass filtering, which may help in noisy settings, and is able to learn relationships between the features via the Combine-Features step.

\subsubsection{Manifold ChebNets}\label{ex: chebnet}

We next construct an MFCN with a more expressive filter bank, analogous to ChebNet \citep{Defferrard2018}. One such approach is to constrain each filter $w_{j,k}(\mathcal{L})$ to be a polynomial of $\mathcal{L}$ parameterized in the Chebyshev basis. However, a potential drawback of this approach is that the eigenvalues of $\mathcal{L}$ are unbounded. Thus, we instead consider polynomials of the heat kernel $\mathcal{P}_t=e^{-t\mathcal{L}}$ for some fixed value of $t$. We note that $\mathcal{P}_t$ has the same eigenfunctions as $\mathcal{L}$ and that its eigenvalues are given by $\omega_k=e^{-t\lambda_k}$. Therefore, polynomials of $\mathcal{P}_t$ are still spectral filters of the form \eqref{eqn: specconvcontinuum}. 

We now design a network where the filters take the form $p_{j,k}(\mathcal{P}_t)$, where each $p_{j,k}$ is a polynomial. We omit the Combine-Features step by setting each of the matrices $\Theta^{(\ell,j)}:=(\theta_{i,k}^{(\ell,j)})_{1\leq i,k\leq C_\ell}$ equal to the identity, and perform a simple combination of the filters $\alpha_{j,i}^{(\ell,k)}=1$, hence obtaining the layerwise update rule $f^{(\ell+1)}_k=\sigma\left(\sum_{i=1}^{C_\ell} p_{i,k}(\mathcal{P}_t)f_i^{(\ell)}\right).$ As in ChebNet, we write each of the $p_{i,k}$ in the Chebyshev basis and treat the basis coefficients as trainable parameters. (Since the $\omega_k$ lie in the interval $[0,1]$,  one should either expand the polynomials in the shifted Chebyshev basis for functions defined on $[0,1]$ or alternatively replace $\mathcal{P}_t$ by $2\mathcal{P}_t-1$ so that its eigenvalues lie in $[-1,1]$.) Since the resulting network has more flexible filters it may be able to learn more general functions than the GCN style network considered in the previous example. 

\subsubsection{Hand-crafted Scattering Networks}\label{ex: scat}

The manifold scattering transform \citep{perlmutter:geoScatCompactManifold2020,chew2022manifold} is a wavelet-based method for deep learning on manifolds inspired by analogous constructions for Euclidean data such as images \citep{mallat:scattering2012, czaja:timeFreqScat2017, nicola2022stability, bruna2013invariant, grohs:cnnCartoonFcns2016, wiatowski:mathTheoryCNN2018} and for graphs \citep{gama:diffScatGraphs2018,gama:stabilityGraphScat2019,zou:graphCNNScat2018,gao:graphScat2018,perlmutter2019understanding, bodmann2022scattering}.
Like the original Euclidean scattering transform \citep{mallat:firstScat2010, mallat:scattering2012}, it is a hand-crafted model, enabling analysis of the stability and invariance properties of deep neural networks. 

To understand the manifold scattering transform via our MFCN framework, we omit both the Combine-Features and the Combine-Filters steps
     and consider a family of wavelet filters $\{w_j(\lambda)\}_{j=1}^J$, where each $w_j$ is chosen to be a band-pass filter (a function whose support is concentrated in an interval) such as $$w_j(\lambda)=e^{-2^{j-1}\lambda}-e^{-2^j\lambda}
    .$$ The hidden representation of each input $f_k$ in the $\ell$-th layer takes the form
$$U[j_1,\ldots,j_\ell]f_k=\sigma(W_{j_\ell}\sigma(W_{j_{\ell-1}}\sigma(\ldots W_{j_1}f_k))\ldots),$$
 where we write $W_j$ in place of $w_j(\mathcal{L})$ for ease of notation.
 
 When used for tasks such as classification, one typically concatenates all of the $U[j_1,\ldots,j_\ell]f_k$ (using each of the possible combinations of $j_1,\ldots,j_\ell$) for $\ell$ up to some $L$, typically $L=2$ or $3$, and this hidden representation is fed into another classifier such as a multilayer perceptron. (Therefore, it may be thought of as modeling a neural network with skip connections.)  Depending on the task, one may also choose to apply low-pass filtering, global averaging, or other aggregations, to the $U[j_1,\ldots,j_\ell]f_k$ before applying the final classifier.

\subsubsection{Learnable Scattering Networks}\label{ex: learnable scattering}

Since the scattering transforms discussed above are hand-crafted networks, they are naturally well-suited for unsupervised learning or low-data environments (when paired with an auxiliary classifier or regressor).  However, in many settings, the hand-crafted design can limit the power of such networks. Therefore, for both Euclidean data and graphs, there have been a variety of works that incorporating learning into the scattering framework.

In the Euclidean setting, \citet{oyallon:scalingScattering2017} created a network that acts as a hybrid of the scattering transform and a CNN using predesigned, wavelet filters in some layers and learnable filters in others. Subsequent work by \citet{zarka2020separation} introduced learning by allowing the network to learn linear combinations of the wavelets. Notably, this is the inspiration for our Combine-Filters step. To construct an analogous MFCN, one may utilize a predesigned (wavelet) filter bank $\{w_j(\mathcal{L})\}_{j=1}^J$ which is shared across all channels, omit the combine step, and let the $\alpha_{j,i}^{(\ell,k)}$ be learnable. (Traditionally, scattering networks have used $\sigma(x)=|x|$ as the activation function, but one could readily use other choices instead.)

In the graph setting, \citet{wenkel2022overcoming} incorporated learning into the scattering framework by utilizing predesigned wavelet filters, but with learnable matrices that combined the input features (along with a few other features to boost performance). In a different approach,  \citet{tong2022learnable} sought to relax the graph scattering transform by replacing dyadic wavelets, e.g., $w(\lambda)=\lambda^{2^{j-1}}-\lambda^{2^j}$, with more general wavelets $w(\lambda)=\lambda^{t_{j-1}}-\lambda^{t_j}$, where $t_j$ is increasing sequence of scales $t_j$ which is learned  via a differentiable selector matrix. (The wavelets from \citet{tong2022learnable} are implemented in terms of a lazy random walk matrix $\mathbf{P}$, whose eigenvalues lie in $[0,1]$. This matrix can be thought of as a spectral filter since it can be written as a polynomial of a normalized graph Laplacian.)

To obtain an analogous MFCN, we can omit the Combine-Features step, and set $a_j(\lambda)=e^{-j\lambda}$ for $0\leq j \leq J$. These filters diffuse the input signal over the manifold at different time-scales. We could then learn relationships between the diffusion scales via combinations of the filters $a_j(\lambda)$ (where the filter combinations utilized in \citet{tong2022learnable} have a certain structure that encourages the network to behave in a wavelet-like manner).

\subsection{Implementation on Point Clouds}\label{sec: MFCN discrete implementation}

In many applications of interest, one does not know the manifold $\mathcal{M}$. 
Instead, as discussed in Section \ref{sec: background manifold learning}, one is given access to finitely many sample points $x_1,\ldots,x_n\in\mathbb{R}^D$ and makes the modeling assumption that these sample points satisfy the manifold hypothesis, i.e., that they lie upon (or near) an unknown $d$-dimensional Riemannian manifold for some $d\ll D$. In this setup, it is non-trivial to actually implement a neural network since one does not have global knowledge of the manifold. 

We will use an approach based on the manifold learning methods discussed in Section \ref{sec: background manifold learning} where we construct a data-driven graph $G_n$, whose vertices are the sample points $x_1,\ldots,x_n$, and use the eigenvectors and eigenvalues of the graph Laplacian $\mathbf{L}_n$ to approximate the eigenfunctions and eigenvalues of the manifold Laplacian $\mathcal{L}$. There are numerous methods for constructing $G_n$, but here we will focus on $k$-NN graphs and $\epsilon$-graphs. (As we shall discuss below, the precise form of the manifold Laplacian $\mathcal{L}$ will be different for these two graph constructions.)

\begin{definition}[$\epsilon$-graphs]\label{ex: eps}
We let $\epsilon>0,$ and place an edge between two vertices $x_i$ and $x_j$ if $\|x_i-x_j\|_2<\epsilon$. We then use a weighting function $\eta:[0,\infty)\rightarrow [0,\infty)$ to assign edge weights.
Following the lead of \citet{Calder2019}, we assume that $\eta$ is a non-increasing function such that $\eta(1/2)>0=\eta(1)$ and will also assume that $\eta$  is Lipschitz continuous on $[0,1]$. We then constructed a weighted adjacency matrix by 
$$
[\mathbf{A}_{n,\epsilon}]_{i,j}=\eta\left(\frac{\|x_i-x_j\|_2}{\epsilon}\right),
$$
and let $\mathbf{D}_{n,\epsilon}$ be the corresponding degree matrix. 
The $\epsilon$-graph Laplacian is then given by 
$$
\mathbf{L}_n=\mathbf{L}_{n,\epsilon}=\frac{1}{c_\eta n\epsilon^{d+2}}\left(\mathbf{D}_{n,\epsilon}-\mathbf{A}_{n,\epsilon}\right),
$$
where 
$c_\eta = \int_{\mathbb{R}^d} |y(1)|^2 \eta(\|y\|_2)dy,$
 is a normalizing constant needed in order to ensure that the eigenvalues of $\mathbf{L}_{n,\epsilon}$ converge to the correct limit (and $y(1)$ is the first coordinate of the point $y\in\mathbb{R}^d$). 
\end{definition}

The graph Laplacians of $\epsilon$-graphs are sparse by construction, and their sparsity is indirectly controlled by the length scale parameter $\epsilon$. To more directly control the sparsity of the graph Laplacian in an adaptive manner without specifying a length scale, one may also consider $k$-NN graphs. 

\begin{definition}[$k$-NN graphs]\label{ex: knn}
For a positive integer $k$, symmetric $k$-Nearest Neighbor ($k$-NN) graphs are constructed by placing an edge between $x_i$ and $x_j$ if $x_j$ is one of the $k$ closest  points to $x_i$ (with respect to the Euclidean distance) or if $x_i$ is one of the $k$ closest points to $x_j$.
To assign edge weights, we again follow the lead of \citet{Calder2019}, and let $\epsilon_k(x_i)$ denote the distance from $x_i$ to its $k$-th closest neighbor. We then let  $r_k(x_i,x_j) := \max\{\epsilon_k(x_i),\epsilon_k(x_j)\}$, and define a weighted adjacency matrix by
$$[\mathbf{A}_{n,k}]_{i,j} = \eta \left ( \frac{|x_i - x_j|}{r_k(x_i,x_j)}\right ),$$
where we set $[\mathbf{A}_{n,k}]_{i,j}=0$ if there is no edge between $x_i$ and $x_j$, and we make the same assumptions on $\eta$ as in the $\epsilon$-graph case.  The $k$-NN graph Laplacian is then given by 
$$
\mathbf{L}_n=\mathbf{L}_{n,k}=\frac{1}{c_\eta n}\left(\frac{nc_d}{k}\right)^{1+2/d}\left(\mathbf{D}_{n,k}-\mathbf{A}_{n,k}\right),
$$
where $c_d$ is the volume of the $d$-dimensional Euclidean unit ball, $c_\eta$ is the same as in Definition \ref{ex: eps}, and $\mathbf{D}_{n,k}$ again the degree matrix corresponding to $\mathbf{A}_{n,k}$.
\end{definition}

Below, in Section \ref{sec: convergence}, we will recall a result from \citet{Calder2019} which shows that if $G_n$ is a graph constructed as in Definition \ref{ex: eps} or \ref{ex: knn} the the eigendecomposition of $\mathbf{L}_n$ will converge to the eigendecomposition of a manifold Laplacian $\mathcal{L}$. However, we note that the limiting operator is not the same for both graph constructions. Specifically, in the case where $G_n$ is an $\epsilon$ graph, the limiting operator is $\mathcal{L}f=-\frac{1}{2\rho}\text{div}(\rho^2\nabla f)$, whereas if $G_n$ is a $k$-NN graph, then the limiting operator is $\mathcal{L}f=-\frac{1}{2\rho}\text{div}(\rho^{1-2/d}\nabla f)$.
 
For an input signal $f_k$, we will let $\mathbf{x}_k=P_nf_k$ be the vector obtained by evaluating $f_k$ at the the data points $x_i$ and performing appropriate normalization. (See Equation \ref{eqn: Pn}.)  Given the graph Laplacian $\mathbf{L}_n$, constructed as in either Definition \ref{ex: eps} or \ref{ex: knn}, we let $\{\phi_i^n\}_{i=1}^n$ be an orthonormal  basis  of eigenvectors, 
$   \mathbf{L}_n \phi_i^n = \lambda_i^n \phi_i^n,$ $0=\lambda_1^n\leq \lambda_2^n\leq \ldots\leq\lambda_n^n$, and expand each $\mathbf{x}_k$ in the Fourier basis    $$
\mathbf{x}_k=\sum_{i=1}^n \widehat{\mathbf{x}_k}(i) \phi_i^n, \quad \widehat{\mathbf{x}_k}(i)=\langle \mathbf{x}_k,\phi^n_i\rangle_2.
$$

In our implementation, we will assume  either that each of the input signals $f$ or each of the spectral filters $w$ are $\kappa$-bandlimited as defined below.

\begin{definition}[Bandlimited functions and bandlimited spectral filters] Let 
$f\in L^2(\mathcal{M})$, let $w(\mathcal{L})$ be a spectral filter, and let $\kappa$ be a positive integer.  We say that $f$ is $\kappa$-bandlimited if $\widehat{f}(i)=0$ for all $i>\kappa$. Similarly, $w(\mathcal{L})$ is said to be $\kappa$-bandlimited if $w(\lambda_i)=0$ for all $i>\kappa.$
\end{definition}
 
We then define a discrete approximation of $w(\mathcal{L})$  by 
\begin{equation}\label{eqn: filtering discrete}
w(\mathbf{L}_n)\mathbf{x}_k=\sum_{i=1}^\kappa w(\lambda^n_i) \widehat{\mathbf{x}_k}(i) \phi^n_i.
\end{equation}
We may then implement an MFCN layer, but with \eqref{eqn: filtering discrete} in place of the filtering step. 
Specifically, we assume we have an  $n\times C$ data matrix $\mathbf{X}=(\mathbf{x}_1,\ldots,\mathbf{x}_C)$, where $\mathbf{x}_k=P_nf_k$, where as before, $P_nf\in\mathbb{R}^n$ is the vector defined by $P_nf(i)=\frac{1}{\sqrt{n}}f(x_i)$. We may then apply the following discrete update rules paralleling those theoretically conducted in the continuum.
\begin{align*}
\text{Filter:}&\quad
\tilde{\mathbf{x}}^{(\ell)}_{j,k}=w^{(\ell)}_{j,k}(\mathbf{L}_n)\mathbf{x}^{(\ell)}_k,\quad 1\leq j \leq J_\ell, 1\leq k\leq C_\ell\\
\text{Combine Features:}&\quad
\mathbf{y}_{j,k}^{(\ell)}=\sum_{i=1}^{C_{\ell}}\tilde{\mathbf{x}}^{(\ell)}_{j,i}\theta^{(\ell,j)}_{i,k},\quad 1\leq j\leq J_\ell, 1\leq k \leq C'_\ell\\
\text{Combine Filters:}&\quad
\tilde{\mathbf{y}}^{(\ell)}_{j,k}= \sum_{i=1}^{J_\ell} \alpha^{(\ell,k)}_{j,i}\mathbf{y}_{i,k},\quad 1\leq j\leq J_\ell',1\leq k\leq C'_\ell\\
\text{Activation:}&\quad
\mathbf{z}_{j,k}^{(\ell)}=\sigma\circ \tilde{\mathbf{y}}_{j,k}^{(\ell)},\quad 1\leq j\leq J_\ell, 1\leq k \leq C'_\ell\\
\text{Reshaping:}&\quad
\mathbf{x}^{(\ell+1)}_{(j-1)C_{\ell}+k} = \mathbf{z}^{(\ell)}_{j,k}, \quad 1\leq j\leq J_\ell',1\leq k\leq C'_\ell. 
\end{align*}

We note that, in principle, the Filter step requires the eigendecomposition of the $\mathcal{L}_n$ which bears an $\mathcal{O}(n^3)$ computational cost (if all $n$ eigenvalues are utilized). However, this cost can be reduced by approximating $w_{j,k}(\lambda)$ by a polynomial which will allow one to implement the filters without directly computing an eigendecomposition.

\section{Convergence Analysis}\label{sec: convergence}

 In this section, we will provide a theoretical analysis of the numerical method proposed in Section \ref{sec: MFCN discrete implementation} for implementing an MFCN when one does not have global knowledge of the manifold but merely has access to finitely many data points, i.e., a point cloud. In particular, in Theorem \ref{thm: bound given filter bound short}, we will prove that 
$$
\lim_{n\rightarrow\infty}\|\mathbf{x}_k^{\ell}-P_nf_k^{\ell}\|_2=0,
$$
where $\mathbf{x}^{(\ell)}_k$ is the $k$-th signal in the $\ell$-th layer of the discrete implementation of the network and $P_nf_k^{(\ell)}$ is the projection of the $k$-th continuum signal onto the data points. 
In other words, with sufficiently many data points, the result of our discrete implementation will be nearly the same as if one implemented the entire network on the manifold $\mathcal{M}$, using global knowledge of both $\mathcal{M}$ and the function $F$,
and then discretized the corresponding output. 

In order to facilitate our analysis, in addition to  assuming either that each of the input signals $f$ or each of the spectral filters $w$ are $\kappa$-bandlimited, and we also assume that $w$ is Lipschitz continuous as defined below.

\begin{definition}[Lipschitz Constant]
 We let $A_{\text{Lip}}(w)$ denote the Lipschitz constant of a spectral filter $w$, i.e.,  the smallest constant such that   $$|w(a)-w(b)| \leq A_{\text{Lip}}(w)|a-b|$$ for all $a,b\in[0,\infty)$.
\end{definition}

We first consider the error associated to an individual spectral filter and aim to show that
 $\|w(\mathbf{L}_n)P_nf-P_nw(\mathcal{L})f\|_2$ will converge to zero as $n$ tends to infinity, where $P_n:\mathcal{C}(\mathcal{M})\rightarrow\mathbb{R}^n$ is the normalized evaluation operator defined as in \eqref{eqn: Pn}. Notably, in order to bound  $\|w(\mathbf{L}_n)P_nf-P_nw(\mathcal{L}_\rho)f\|_2$ we must account for three sources of discretization error:
 \begin{enumerate}
     \item The graph eigenvalue $\lambda_i^n$ does not exactly equal the manifold eigenvalue $\lambda_i$. Intuitively, one anticipates that this should yield an error on the order of $\alpha_{i,n}A_{\text{Lip}(w)}$, where $\alpha_{i,n}=|\lambda_i-\lambda_i^n|$.
     \item The graph eigenvector $\phi_i^n$ does not exactly equal $P_n\phi_i$, the discretization of the true continuum eigenfunction. One may expect this to yield errors of the order $\beta_{i,n}$, where $\beta_{i,n}=\|\phi_i^n-P_n\phi_i\|_2$.
     \item The discrete Fourier coefficient $\widehat{\mathbf{x}}(i)$ is not exactly equal to $\widehat{f}(i)$. Since Fourier coefficients are defined in terms of inner products, one expects this error to be controlled by a term $\gamma_n$ which describes how much discrete inner products $\langle P_n f,P_n g\rangle_2$ differ from continuum inner products $\langle f,g\rangle_{L^2(\mathcal{M})}$. 
 \end{enumerate}
Combining these sources of error, and letting $\alpha_n=\max_i\alpha_{i,n},\beta_n=\max_i\beta_{i,n}$, one anticipates that if either $f$ or $w(\mathcal{L}_\rho)$ is $\kappa$-bandlimited, then the total error will be $\mathcal{O}(\kappa(\alpha_nA_{\text{Lip}(w)}+\beta_n+\gamma_n))$.

In order to control the first two sources of error, we recall the following result: 

\begin{theorem}[Theorems 2.4, 2.5, 2.7, and 2.9 of \citet{Calder2019}]\label{thm: recall Calder results}
Let $G_n$ be constructed as either a $\epsilon$- or $k$-NN graph. In the $\epsilon$-graph case assume $\epsilon \sim \left ( \frac{\log(n)}{n}\right )^{\frac{1}{d+4}}$ and let $\mathcal{L}f=-\frac{1}{2\rho}\text{div}(\rho^2\nabla f)$. In the $k$-NN graph case assume that $k \sim \log(n)^{\frac{d}{d+4}} n^{\frac{4}{d+4}}$ and let $\mathcal{L}f=-\frac{1}{2\rho}\text{div}(\rho^{1-2/d}\nabla f)$.  Fix $\kappa>0$ and let $\alpha_{n}=\max_{1\leq i\leq \kappa}|\lambda_i-\lambda_i^n|$ and $\beta_n=\max_{1\leq i\leq \kappa}\|\phi_i^n-P_n\phi_i\|_2$.  Then, with $1-\mathcal{O}(n^{-9})$, 
\begin{equation}
\label{eqn: epsilon graph constants}
    \alpha_n = \mathcal{O}\left ( \frac{\log(n)}{n}\right )^{\frac{1}{d+4}}, \quad \text{and}\quad\beta_n = \mathcal{O}\left ( \frac{\log(n)}{n}\right )^{\frac{1}{d+4}},
\end{equation}
where the implied constants depend on the geometry of the manifold $\mathcal{M}$. 
\end{theorem}

Additionally, in order to estimate the third source of error, arising from the fact that $\widehat{\mathbf{x}}(i)\neq \widehat{f}(i)$, we will prove the following lemma which is a refined version of Lemma 5 of \citet{chew2022geometric}.  It builds upon the previous result in two ways: (i) it allows for the density $\rho$ to be non-constant, and (ii) its proof utilizes Bernstein's inequality rather than Hoeffding's inequality. This leads to a bound depending on the $L^4$ norms of $f$ and $g$ rather than the $L^\infty$ norms.
For a proof, please see Appendix \ref{sec: proof of bernstein}

\begin{lemma}\label{lem: bernstein}
For all $f,g\in\mathcal{C}(\mathcal{M})$, and for sufficiently large $n$, with probability at least $1-\frac{2}{n^9}$ we have
    \[|\langle P_n f, P_n g\rangle_2 - \langle f, g \rangle_{L^2(\mathcal{M})}| \leq 6 \sqrt{\frac{\log(n)}{n}}\|f\|_{L^4(\mathcal{M})}\|g\|_{L^4(\mathcal{M})}.\]
\end{lemma}

Together, Theorem \ref{thm: recall Calder results} and   Lemma \ref{lem: bernstein} enable us to obtain the following result, which establishes the convergence of graph spectral filters $w(\mathbf{\mathbf{L}}_n)$ to their continuum limit $w(\mathcal{L})$.
For a proof, please see Appendix \ref{sec: proof of theorem filter error}.
\begin{theorem}\label{thm: Filter Error short} Let $G_n$ be constructed as in Theorem \ref{thm: recall Calder results} and assume that the high-probability events considered in both Theorem \ref{thm: recall Calder results} and Lemma \ref{lem: bernstein} hold. Let $w:[0,\infty)\rightarrow \mathbb{R}$, $\|w\|_{L^\infty([0,\infty))}\leq 1$,
let $f\in \mathcal{C}(\mathcal{M})$,  and assume that either $f$ or $w(\mathcal{L})$ is $\kappa$-bandlimited for some $\kappa>0$. 
Then, for any continuous function $f$, we have 
\begin{align*}\label{eqn: filter stability with x}
&\|w(\mathbf{L}_n)P_nf-P_nw(\mathcal{L})f\|_2\\&\leq 
\mathcal{O}\left(\left( \frac{\log(n)}{n}\right)^{\frac{1}{d+4}}\right)(A_{\text{Lip}}(w)+1)\|f\|_{L^2(\mathcal{M})}+\mathcal{O}\left(\left(\frac{\log(n)}{n}\right)^{\frac{1}{4}}\right)\|f\|_{L^4(\mathcal{M})},
\end{align*}
where the constants implied  by the big-$\mathcal{O}$ notation depend on the geometry of $\mathcal{M}$. In particular, if $A_{\text{Lip}}(w)$ is finite, we have
$\lim_{n\rightarrow\infty}    \|w(\mathbf{L}_n)P_nf-P_nw(\mathcal{L})f\|_2=0.$
\end{theorem}

\begin{remark}\label{rem: relaxed assumptions long}
Inspecting the proof of Theorem \ref{thm: Filter Error short}, one may note that $A_{\text{Lip}}(w)$ may actually be replaced by the Lipschitz constant on the smallest interval containing all $\lambda_i$ and all $\lambda_i^n$, $1\leq i \leq \kappa$, where $\lambda_i\neq \lambda_i^n$. This means that, if $f$ is bandlimited, our result may be applied to any continuously differentiable function $w$. Moreover, if the graph $G_n$ is connected, we have $\lambda_1=\lambda_1^n=0$ and $\lambda_2,\lambda_2^n>0$. This implies that our theorem can be applied to any $w$ which is continuously differentiable on $(0,\infty)$ even if, for example, $\lim_{t\rightarrow 0^+}w'(t)=+\infty$ (which is the case for certain types of wavelets, such as those considered in \citet{perlmutter:geoScatCompactManifold2020}). Additionally, we note that with minor modifications, results similar to Theorem \ref{thm: Filter Error short} may be obtained for functions or filters that are approximately bandlimited in the sense that either $\sup_{k>\kappa}|w(\lambda_k)|$ or $\sum_{k>\kappa}|\widehat{f}(k)|^2$ are sufficiently small. In these cases, we will have 
$
\limsup_{n\rightarrow\infty} \|w(\mathbf{L}_n)\mathbf{x}-P_nw(\mathcal{L})f\|_2 \leq \sup_{k>\kappa}|w(\lambda_k)|\|f\|_{L^2(\mathcal{M})}$
 or $
\limsup_{n\rightarrow\infty} \|w(\mathbf{L}_n)\mathbf{x}-P_nw(\mathcal{L})f\|_2 \leq \|w\|_\infty\left(\sum_{k>\kappa}|\widehat{f}(k)|^2\right)^{1/2}$. In particular, results similar to  Theorem \ref{thm: Filter Error short} may be obtained for filters such as $w_t(\lambda)\coloneqq e^{-t\lambda}$, which correspond to the heat kernel.
\end{remark}

We next apply Theorem \ref{thm: Filter Error short} to obtain the following result, which guarantees that the numerical method introduced in Section \ref{sec: MFCN discrete implementation} for implementing MFCNs on point clouds will converge as the number of data points tends toward infinity in the sense that $\lim_{n\rightarrow\infty}\|\mathbf{x}_k^{(\ell)}-P_n f_k^{(\ell)}\|_2=0$. In other words, for large $n$, the output of the $k$-th channel in the $\ell$-th layer of the approximating GNN,  $\mathbf{x}_k^{(\ell)}$, will be nearly identical to the result one would have obtained if the entire network were implemented in the continuum and then subsampled at the very end.
Notably, we see that if one one normalizes the weights of the network so that $A_1^{(\ell)}=A_2^{(\ell)}=1$, then the rate of convergence will be independent of the number of filters per layer and is linear in the depth of the network. This is in contrast to results in \citet{chew2022geometric}, and \citet{wang2023geometric}, which exhibit polynomial dependence on the number of filters per layer and exponential dependence on the network depth. 
\begin{theorem}\label{thm: bound given filter bound short}
Let $f \in \mathcal{C}(\mathcal{M})$. Consider an MFCN where all of the filters or input signals, as well as the construction of the graph $G_n$, satisfy the assumptions of Theorem \ref{thm: Filter Error short}.  
Let $A_1^{(\ell)}=\max_{j,k}(|\sum_{i=1}^{C_{\ell}} |\theta_{i,k}^{(\ell,j)}|),$ $A_2^{(\ell)}=\max_{j,k}(\sum_{i=1}^{J_{\ell}} |\alpha_{j,i}^{(\ell,k)}|)$,  and assume that every activation function $\sigma$ is non-expansive, i.e. $|\sigma(x)-\sigma(y)|\leq |x-y|$.
Then, 
\begin{align*}&\|\mathbf{x}_k^{\ell}-P_nf_k^{\ell}\|_2\\\leq& \left(\sum_{i=0}^{\ell-1}\hspace{-.01in}\prod_{j=i}^{\ell-1}\hspace{-.01in} A_1^{(j)}A_2^{(j)}\hspace{-.025in}\right)\hspace{-.075in}\left(\hspace{-.035in}\mathcal{O}\hspace{-.025in}\left(\hspace{-.035in}\left(\hspace{-.025in} \frac{\log n}{n}\hspace{-.01in}\right)^{\frac{1}{d+4}}\right)(A_{\text{Lipmax}}(w)\hspace{-.03in}+\hspace{-.03in}1)\|f\|_{L^2(\mathcal{M})}\hspace{-.025in}+\hspace{-.025in}\mathcal{O}\hspace{-.025in}\left(\hspace{-.045in}\left(\frac{\log n}{n}\right)^{\frac{1}{4}}\right)\hspace{-.035in}\|f\|_{L^4(\mathcal{M})}\hspace{-.035in}\right),\end{align*}
where $A_{\text{Lipmax}}$ is the maximum Lipschitz constant amoungst all of the filters used in the network.
In particular, if $A_{\text{Lipmax}}(w)$ is finite, we have
$\lim_{n\rightarrow\infty}\|\mathbf{x}_k^{\ell}-P_nf_k^{\ell}\|_2=0.$
\end{theorem}

We provide a proof of Theorem \ref{thm: bound given filter bound short} in Appendix \ref{sec: proof of network theorem}. Additionally, as we discuss further in Appendix \ref{sec: proof of network theorem}, one can also consider generalized MFCNs where, in the filtering step, the spectral convolution operators $w_{j,k}^{(\ell)}(\mathcal{L})$ are replaced with generic linear operators on $L^2(\mathcal{M})$. In this work, we do not propose a numerical method for implementing other types of filtering operators. However, we do, in Appendix \ref{sec: proof of network theorem}, state a more general version of Theorem \ref{thm: bound given filter bound short} which explicitly relates the convergence rate of the entire MFCN to the convergence rate of the individual filters. Therefore, if one were to produce convergence guarantees for non-spectral filters, then our theory would immediately provide convergence results for the corresponding MFCNs. Similarly, convergence results for other classes of spectral filters or functions that are not necessarily bandlimited, such as the FDT filters considered in \citet{wang2022convolutional} and \citet{wang2023geometric}, will lead to immediate corollaries for MFCNs constructed with such filters.

\section{Numerical experiments}\label{sec: results}

Below in Section \ref{sec: convergence_results}, we will provide a numerical demonstration of Theorem \ref{thm: Filter Error short} showing the convergence of spectral filters on the sphere (where we have access to ``ground truth" due to the availability of closed form expressions for the Laplacian eigenfunctions). Then, in Section \ref{sec: classification and regression}, we conduct experiments using several MFCN models on real and synthetic data sets. Specifically, in Section \ref{sec: ellipsoids_results}, we consider regression tasks for functions defined on the surface of a two-dimensional ellipsoid (embedded in eight-dimensional space), testing the ability of various models to learn both individual eigenfunctions as well as random bandlimited functions. Then, in Section \ref{sec: melanoma_results}, we will showcase the utility of our method on real-world data by applying it to a high-dimensional data set derived from melanoma patients. For these later experiments, we will also propose a new method, Infogain, which we use to help design a wavelet-based MFCN model. We provide further details on Infogain in Section \ref{sec: info}. Code to reproduce our experiments can be found at \url{https://github.com/dj408/mfcn}.

\subsection{Convergence experiments}\label{sec: convergence_results}

We first empirically demonstrate the convergence of spectral filters on synthetic data, complementing our theoretical analysis in Theorem \ref{thm: Filter Error short}. In these experiments, we focus on the two-dimensional unit sphere embedded in $\mathbb{R}^3$, since the eigenfunctions and eigenvalues of the Laplacian on $\mathbb{S}^2$ are known in closed form (i.e., the eigenfunctions are spherical harmonics).

We consider a simple function $f$ chosen to be the sum of two Laplacian eigenfunctions, $f=Y^0_1+Y^0_2$, where $Y^i_j$ denotes the $i$-th spherical harmonic of degree $j$. We sample $n$ points uniformly from the sphere, evaluate $f$ at these points, and construct a $k$-NN graph $G_n$ as described in Definition \ref{ex: knn}, with $k$ chosen in accordance with Theorem \ref{thm: recall Calder results}. We then apply a spectral filter\footnote{When computing the eigendecomposition, we use only the first $64$ eigenpairs, as $w(\lambda)=e^{-\lambda}$ takes negligible values at the higher eigenvalues.} $w(\lambda)=e^{-\lambda}$ so that $w(\mathcal{L})$ is the heat kernel $e^{-\mathcal{L}}$. 

In Figure \ref{fig:knn_spectral_filter_converg_plot}, we report the discretization errors $\|w(\mathbf{L}_n) P_n f - P_n w(\mathcal{L}) f\|_2$ over 10 trials for $n$ ranging from $2^6$ to $2^{14}$. We also report the convergence of the first two distinct, non-zero eigenvalues (corresponding to the first eight eigenvalues counting their multiplicity) 
in Figure \ref{fig:knn_eigenvals_converg_plot}. In the case of uniform sampling on the two-dimensional unit sphere, the manifold Laplacian $\mathcal{L}f=-\frac{1}{2\rho}\text{div}(\rho^{1-2/d}\nabla f)$ reduces to  $-2\pi\Delta$, where $-\Delta=-\text{div}\circ\nabla$ is the negative Laplace-Beltrami operator, and we see that the numerical eigenvalues converge to the true values of $2\pi \ell(\ell+1)$, for $\ell=1,2$. Additionally, in Appendix \ref{appendix: convergence}, we conduct similar experiments in the setting where $G_n$ is constructed as an $\epsilon$-graph (per Definition \ref{ex: eps}).

\begin{figure*}[hbt!]
    \begin{center}
    \includegraphics[width=1.0\textwidth]{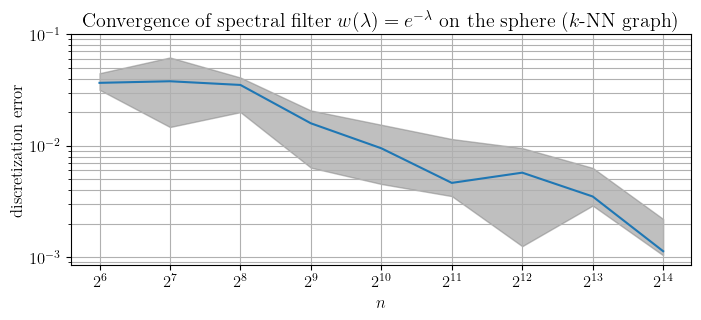}
    \caption{Discretization error for spectral filter $w(\lambda)=e^{-\lambda}$ applied to the sum of two spherical harmonics, for a $k$-NN graph construction. The median error of 10 runs is shown in blue, against a gray band of the 25th- to 75th-percentile error range.}
    \label{fig:knn_spectral_filter_converg_plot}
    \end{center}
\end{figure*}

\begin{figure*}[hbt!]
    \begin{center}
    \includegraphics[width=1.0\textwidth]{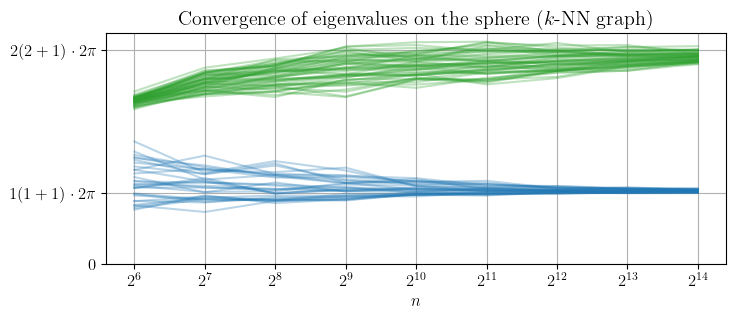}
    \caption{Convergence of first eight, non-zero eigenvalues on the sphere, for a $k$-NN graph construction, all 10 runs combined. The blue lines are the first three eigenvalues that converge to the same limit (since the first non-zero eigenvalue of the spherical Laplacian has multiplicity three). Similarly, the green lines are the next five eigenvalues.}
    \label{fig:knn_eigenvals_converg_plot}
    \end{center}
\end{figure*}
\subsection{Classification and regression experiments}\label{sec: classification and regression}

Below in Sections \ref{sec: ellipsoids_results} and \ref{sec: melanoma_results}, we will assess the effectiveness of several MFCN models for regression and classification tasks. A high-level of illustration of MFCNs, as applied to these tasks, is given in Figure \ref{fig:architecture_cartoon}.

We first consider MFCNs inspired by the extremely popular GCN introduced in \citet{kipf:classGCNN2017} (discussed earlier in Section \ref{ex: mcn}). To do this, in each layer, we use a single low-pass  filter in our Filter step and a trainable matrix in our Combine-Features step. (Since there is only one filter, we omit the Combine-Filters step). For our low-pass filter, we will fix $t>0$ and use the heat kernel $\mathcal{P}_{t}=e^{-t\mathcal{L}}$, i.e., $\mathcal{P}_t=a_t(\mathcal{L})$, where $a_t(\lambda)=e^{-t\lambda}$. (In our experiments, we will consider $t=1/2$ and $t=1$.)  Additionally, for increased computational efficiency, we will also consider a variation that approximates the heat kernel by the lazy random walk diffusion operator $\mathbf{P}_n= \frac{1}{2}(\mathbf{I}_n + \mathbf{A}_n\mathbf{D}_n^{-1})$. 
Since these networks utilize low-pass filters, we will refer to them as `MFCN-low-pass-spectral-0.5', `MFCN-low-pass-spectral-1.0' and `MFCN-low-pass-approx', respectively (where $0.5$ and $1.0$ indicate the value of $t$). 

We next consider wavelet-based MFCN models inspired by the scattering transforms discussed in Sections \ref{ex: scat} and \ref{ex: learnable scattering}. 
We will utilize  a wavelet filter bank, based on Diffusion Wavelets \citep{coifman:diffWavelets2006}, of the form $\{\{w_j(\mathcal{L})\}_{j=0}^N\cup \{a_J(\mathcal{L}\}\}$. As with the low-pass models, we fix $t>0$ and consider the heat kernel $\mathcal{P}_{t}=e^{-t\mathcal{L}}$. We then define $w_0(\mathcal{L})=\mathcal{I} - \mathcal{P}_t$, $w_j(\mathcal{L})=\mathcal{P}_t^{2^{j-1}} - \mathcal{P}_t^{2^{j}}
$ for $1\leq j \leq J$, and $a_J(\mathcal{L}\}=\mathcal{P}_t^{2^{J}}$. (We note that we may write these wavelets as spectral filters $w_j(\mathcal{L})$ and $a_J(\mathcal{L})$ by setting 
 $w_{0}(\lambda) = 1 - e^{-t\lambda}$, $w_{j}(\lambda) = e^{-2^{j-1}t\lambda} - e^{-2^{j}t\lambda}$ for $1 \leq j \leq J$,   and  by $a_{J}(\lambda) = e^{-2^{Jt\lambda}}$.)  
Similar to the low-pass case, we will also utilize wavelets which approximate the heat kernel by the lazy random walk diffusion operator $\mathbf{P}_n$. 
This yields wavelets of the form
$w_j(\mathbf{P}_n) = \mathbf{P}_n^{2^{j-1}} - \mathbf{P}_n^{2^{j}}$, $1\leq j\leq J$, together with $w_0(\mathbf{P}_n)=\mathbf{I}_n-\mathbf{P}_n$ and $a_J(\mathbf{P}_n)=\mathbf{P}_n^{2^J}.$ 

We will refer to our MFCNs defined using wavelets as `MFCN-wavelet-spectral-0.5', `MFCN-wavelet-spectral-1.0' and `MFCN-wavelet-approx', respectively. In order to increase the expressive power of these networks, we will utilize learnable matrices in both the Combine-Features and Combine-Filters step except where otherwise noted. Also, we note that, similar to \citet{Xu2024}, but unlike most traditional scattering networks, we include the low-pass filter $a_j(\mathcal{L})$ (or $a_j(\mathbf{P}_n)$) 
 in our filter bank in each layer, and do not utilize skip connections. 

Additionally, we observe that on the melanoma dataset considered in Section \ref{sec: melanoma_results} below, the predesigned choice of dyadic scales $2^j$ may be overly rigid and hinder performance. Therefore, we also consider generalized diffusion wavelets, inspired by \citet{tong2022learnable}, of the form $w_j(\mathbf{P}_n) = \mathbf{P}_n^{t_{j}} - \mathbf{P}_n^{t_{j+1}}$, $a_J(\mathbf{P}_n)=\mathbf{P}_n^{t_J}$, where $t_0<t_1<t_2\ldots <t_J$ is an increasing sequence of diffusion scales.
However, rather than selecting the scales using the method from \citet{tong2022learnable}, we will introduce Infogain, a method for choosing these scales based on the KL divergence 
(details in Section \ref{sec: info}). We will refer to MFCNs using these wavelets as `MFCN-wavelet-Infogain.' Notably, unlike the method in \citet{tong2022learnable}, which learns the scales from labeled training data, our method of scale selection is unsupervised which allows it to be effective in low-data environments such as the melanoma data set. Furthermore, Infogain will, in general, select different scales for each input channel, unlike the previous method.  Additionally, we note that $\mathbf{P}_n$ is a sparse matrix (as long as the graph $G_n$ is sparse). This allows us to compute $\mathbf{P}_n^{t} \mathbf{x}$ efficiently via recursive sparse matrix multiplication on a vector $\mathbf{x}$ and increases the computational efficiency of diffusion wavelets, especially on large data sets. 

We compare our MFCN models against several popular GNNs (applied to the graph $G_n$ which we construct from the data points): Graph Convolutional Network (GCN) \citep{kipf2016semi}, Graph Attention Network (GAT) \citep{velivckovic2017graph}, Graph Isomorphism Network (GIN) \citep{xu2018how}, and GraphSAGE \citep{hamilton2017inductive}. For all models, we construct $G_n$ as a $k$-NN graph as described in Definition \ref{ex: knn} based on the observation that $k$-NN graphs are the choice most often used in practice (as has previously been noted in \citet{Calder2019}). Spectral filters are implemented using the first twenty non-trivial eigenvectors of $\mathbf{L}_n$, $\phi_2^n,\ldots,\phi^{21}_n$. Further experimental details, such as hyperparameter settings, can be found in Appendix \ref{appendix:experimental_details}. 

\begin{figure*}[hbt!]
    \begin{center}
    \includegraphics[width=1.0\textwidth]{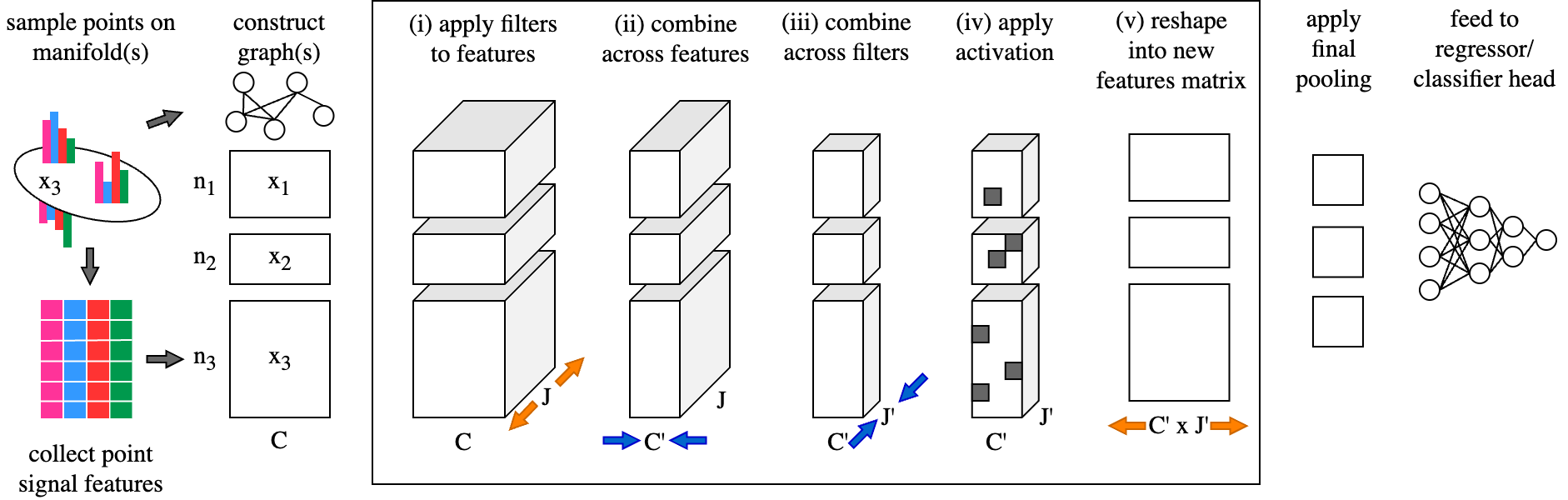}
    \caption{High-level illustration of MFCNs as applied to point-cloud data. An MFCN may have multiple cycles of steps (i) to (v) (boxed). Note that step (i) is typically combinatorially expansive (i.e. transforms the raw feature matrix $\mathbf{X} \in \mathbb{R}^{N \times C} \rightarrow \mathbb{R}^{N \times C \times J}$, where $N$ is the total number of nodes across all batched graphs, $C$ is the number of channels/features, and $J$ is the number of filters). However, steps (ii) and (iii) may be expansive or reductional (as drawn), depending on the whether the number of input features is less than or greater than the number of output features specified as hyperparameters. After filter-combine cycles, final pooling layers can be applied in node-wise and/or feature-wise fashion (e.g. as linear, scattering moments, max, mean, or `top-k' pooling layers, etc.). Finally, the resulting features may be fed into a regressor or classifier head, depending on the learning task.}
    \label{fig:architecture_cartoon}
    \end{center}
\end{figure*}

\subsubsection{Ellipsoid point regression task}\label{sec: ellipsoids_results}

We test the ability of MFCNs and other models to learn functions defined on the surface of two-dimensional ellipsoids, embedded in eight-dimensional space under both noiseless and noisy sampling conditions. We uniformly sample $1024$ points lying on an ellipsoid lying in three-dimensional space defined by $x^2/a^2 + y^2/b^2 + z^2/c^2 = 1$ for Euclidean coordinates $x$, $y$, and $z$ with $a = 3, b = 2, c = 1$. We then embed this ellipsoid in eight-dimensional space by first applying the map $(x,y,z)\mapsto (x,y,z,0,0,0,0,0)$ to the sampled points and then applying a random eight-dimensional rotation.
 Next, we construct an unweighted symmetric $k$-NN graph from the embedded points, using the method described in Definition \ref{ex: knn} and calculate its graph Laplacian $\mathbf{L}_n$, as well as the corresponding eigendecomposition $\mathbf{L}_n\phi^n_i=\lambda_i^n\phi_i^n$.

As a target function, we construct a random bandlimited function from the first 20 nontrivial eigenvectors, $f = \sum_{i=2}^{21} a_i \phi_{i}^{n}$, where the $a_i$ are  sampled randomly from a uniform distribution on $(-1, 1)$. (We don't use the first eigenvectore because it is constant.) Figure \ref{fig:ellip_rand_bandlimit_signal} shows an example of such a signal on an ellipsoid (in three-dimensional  ambient space before we map it to $\mathbb{R}^8$). Additionally, we provide further regression experiments where the target function is an individual eigenvector in Appendix \ref{appendix: ellipsoid}.

Our goal is to regress (estimate) the signal, i.e., the target function, at points not in the training data. We assess this via five-fold cross-validation. That is, we  train the model on $80\%$ of the training points and use the other $20\%$ as a validation set. We then repeat this five times, so that each point gets an opportunity to belong to the validation set, and average performance metrics over all five validation scores. We then repeat this experiment ten times, each time generating a new ellipsoid and new bandlimited signal and then running five-fold cross-validation.

\begin{figure*}[hbt!]
    \begin{center}
    \includegraphics[width=1.0\textwidth]{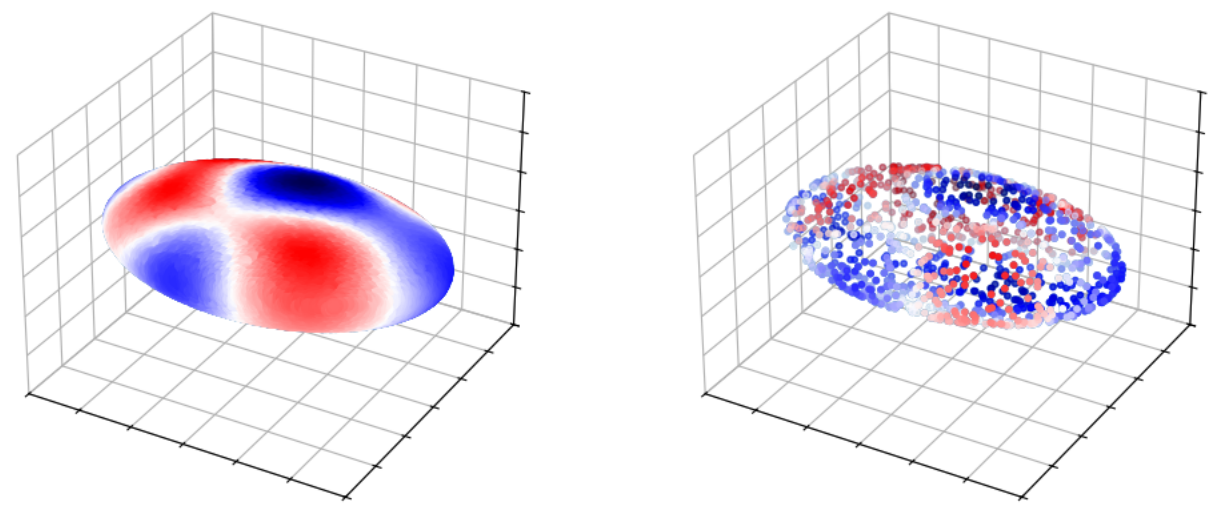}
    \caption{Illustration of a random bandlimited function constructed from the first 20 nontrivial eigenvectors of an ellipsoid (as estimated from the Laplacian of a $k$-NN graph) (left), and evaluated pointwise in a uniform sample of 1024 points (right).}
    \label{fig:ellip_rand_bandlimit_signal}
    \end{center}
\end{figure*}

Next, we repeat the above experiment, but with the addition of multivariate Gaussian noise on the sample points. This simulates the setting where the manifold hypothesis is only \emph{approximately true} in the sense that the data points merely lie near the manifold $\mathcal{M}$ rather than exactly on it, which is the case in most real-world applications. To do so, we use data points $\widetilde{x}_i = x_i + \mathbf{\epsilon}_i$, where $x_i\in\mathcal{M}$ and $\mathbf{\epsilon}_i$ is an i.i.d. normal random variable, $\mathbf{\epsilon}_{i} \sim \mathcal{N}(\mathbf{0}, \sigma_{\epsilon} \mathbf{I}_{8 \times 8})$. We calibrate the noise level $\sigma_{\epsilon}$ to induce sample points to lose an average of 10\% of their true $k$ nearest neighbor assignments. With $n=1024$, we find that this yields $\sigma_{\epsilon} \approx \frac{1}{40 \sqrt{2}}$.

Results of these experiments are presented in Table \ref{table:ellipsoids_combination_signal_results}.
We see that all of the wavelet-based MFCNs perform well, achieving an $R^2$ of at least $0.88$ in the noiseless case and at least $0.69$ in the noisy case. The low-pass-based MFCNs are a bit more inconsistent. In particular, while MFCN-low-pass-spectral-0.5 is the top performing method, MFCN-low-pass-1.0 fails utterly, achieving an $R^2$ of $0.00$ in both the noisy and the noiseless settings. This is likely because if the diffusion-time parameter $t$, used in the heat-kernel $\mathcal{P}_t=e^{-t\mathcal{L}}$, is chosen to be too large, the low-pass filtering will result in nearly constant features, paralleling the well-known oversmoothing problem in GNNs \citep{rusch2023survey}. Of the baselines, GraphSage is the top performing model, achieving $R^2$ scores nearly as high as the wavelet-based MFCNs. Interestingly, GIN consistently fails to learn, with a mean $R^2$ score of $0.00$ in both settings. This may be because GIN is primarily designed for tasks such as graph classification. 

\begin{table}[hbt!]
    \caption{Performance metrics on the ellipsoids random signal node regression task, noiseless sampling. Mean of means $\pm$ mean standard deviation (i.e. square root of mean variance) from five-fold cross validations of 10 unique datasets are shown. MSE: mean squared error. Note that GIN and MFCN-low-pass-spectral-1.0 both appear to have failed to learn under the experimental regime. $R^2$ scores $<0$ (occurring for numerical reasons) have been corrected back to $0$. Best $R^2$ and MSE results are \textbf{bolded}.}
    \label{table:ellipsoids_combination_signal_results}
    \begin{tabular}{@{}lcccrr@{}}
    \toprule
    \textbf{Sampling} $\vert$ Model & $R^2$  & MSE & Sec. per epoch & Num. epochs & Min. per fold \\
    \midrule
    \textbf{Noiseless sampling}\\ 
    \hline
    MFCN-low-pass-spectral-0.5 & $ \mathbf{0.98 \pm 0.01}$ & $ \mathbf{0.0026 \pm 0.0009} $ & $ 0.0044 \pm 0.0113 $ & $ 1786 \pm 338 $ & $ 0.13 \pm 0.02 $ \\
    MFCN-wavelet-spectral-0.5 & $ 0.94 \pm 0.04 $ & $ 0.0065 \pm 0.0040 $ & $ 0.0053 \pm 0.0151 $ & $ 1210 \pm 379 $ & $ 0.11 \pm 0.03 $ \\
    MFCN-wavelet-spectral-1.0 & $ 0.90 \pm 0.06 $ & $ 0.0100 \pm 0.0059 $ & $ 2.0594 \pm 0.0532 $ & $ 1038 \pm 289 $ & $ 35.61 \pm 9.91 $ \\
    MFCN-wavelet-approx-dyadic & $ 0.88 \pm 0.02 $ & $ 0.0059 \pm 0.0013 $ & $ 0.0319 \pm 0.0040 $ & $ 325 \pm 94 $ & $ 0.17 \pm 0.05 $ \\
    GraphSAGE & $ 0.85 \pm 0.02 $ & $ 0.0077 \pm 0.0016 $ & $ 0.0032 \pm 0.0009 $ & $ 712 \pm 386 $ & $ 0.04 \pm 0.02 $ \\
    GAT & $ 0.76 \pm 0.14 $ & $ 0.0116 \pm 0.0058 $ & $ 0.0048 \pm 0.0013 $ & $ 1829 \pm 805 $ & $ 0.15 \pm 0.06 $ \\
    MFCN-low-pass-approx & $ 0.72 \pm 0.12 $ & $ 0.0279 \pm 0.0108 $ & $ 0.0080 \pm 0.0043 $ & $ 3531 \pm 1895 $ & $ 0.47 \pm 0.25 $ \\
    GCN & $ 0.68 \pm 0.15 $ & $ 0.0154 \pm 0.0061 $ & $ 0.0040 \pm 0.0013 $ & $ 3246 \pm 1458 $ & $ 0.22 \pm 0.10 $ \\
    GIN & $ 0.00 \pm 0.00 $ & $ 0.0512 \pm 0.0076 $ & $ 0.0035 \pm 0.0015 $ & $ 282 \pm 199 $ & $ 0.02 \pm 0.01 $ \\
    MFCN-low-pass-spectral-1.0 &   $0.00 \pm 0.00$ &   $0.1125 \pm 0.0123$ &   $1.9483 \pm 0.0539$ &   $11 \pm \;3$ & $0.35 \pm 0.11$\\
    \hline
    \textbf{Noisy sampling}\\
    \hline
    MFCN-low-pass-spectral-0.5 & $\mathbf{ 0.76 \pm 0.04}$ & $ \mathbf{0.0256 \pm 0.0040} $ & $ 0.0050 \pm 0.0194 $ & $ 653 \pm 272 $ & $ 0.05 \pm 0.02 $ \\
    MFCN-wavelet-spectral-0.5 & $ 0.74 \pm 0.05 $ & $ 0.0272 \pm 0.0039 $ & $ 0.0057 \pm 0.0232 $ & $ 540 \pm 253 $ & $ 0.05 \pm 0.02 $ \\
    MFCN-wavelet-spectral-1.0 & $ 0.73 \pm 0.06 $ & $ 0.0282 \pm 0.0043 $ & $ 2.0702 \pm 0.0867 $ & $ 523 \pm 258 $ & $ 18.06 \pm 8.95 $ \\
    MFCN-wavelet-approx-dyadic & $ 0.69 \pm 0.06 $ & $ 0.0337 \pm 0.0066 $ & $ 0.0397 \pm 0.0019 $ & $ 240 \pm 94 $ & $ 0.16 \pm 0.06 $ \\
    GraphSAGE & $ 0.65 \pm 0.07 $ & $ 0.0375 \pm 0.0069 $ & $ 0.0039 \pm 0.0005 $ & $ 704 \pm 276 $ & $ 0.05 \pm 0.02 $ \\
    MFCN-low-pass-approx & $ 0.63 \pm 0.06 $ & $ 0.0392 \pm 0.0067 $ & $ 0.0084 \pm 0.0016 $ & $ 1192 \pm 614 $ & $ 0.17 \pm 0.09 $ \\
    GAT & $ 0.63 \pm 0.08 $ & $ 0.0397 \pm 0.0078 $ & $ 0.0077 \pm 0.0001 $ & $ 1427 \pm 789 $ & $ 0.18 \pm 0.10 $ \\
    GCN & $ 0.53 \pm 0.11 $ & $ 0.0506 \pm 0.0121 $ & $ 0.0060 \pm 0.0005 $ & $ 3206 \pm 1914 $ & $ 0.32 \pm 0.19 $ \\
    GIN & $ 0.00 \pm 0.00 $ & $ 0.1126 \pm 0.0142 $ & $ 0.0040 \pm 0.0025 $ & $ 265 \pm 191 $ & $ 0.02 \pm 0.01 $ \\
    MFCN-low-pass-spectral-1.0 &   $0.00 \pm 0.00$ &   $0.1124 \pm 0.0142$ &   $1.9374 \pm 0.0000$ &   $\;8 \pm \;2$ & $0.26 \pm 0.08$\\
    \botrule
    \end{tabular}
\end{table}

\begin{table}[hbt!]
    \caption{Performance metrics on the melanoma manifold classification task. Mean of means $\pm$ mean standard deviations from five repetitions of 10-fold cross validations are shown. F1 scores are based on the (minority) positive class representing patients who did respond to treatment. Results are sorted by descending mean accuracy (rounding may show false ties); best accuracy and F1 scores are \textbf{bolded}.}
    \label{table:melanoma_results}
    \begin{tabular}{@{}lccrrr@{}}
    \toprule
    Model & Accuracy  & F1 score & Sec. per epoch & Num. epochs & Min. per fold\\
    \midrule
    MFCN-wavelet-approx-Infogain & $\mathbf{ 0.81 \pm 0.14}$ & $\mathbf{ 0.71 \pm 0.30}$ & $ 1.0089 \pm 0.0128 $ & $ 49 \pm 31 $ & $ 0.83 \pm 0.51 $ \\
    GIN & $ 0.78 \pm 0.17 $ & $ 0.69 \pm 0.28 $ & $ 0.0666 \pm 0.0144 $ & $ 53 \pm 41 $ & $ 0.06 \pm 0.05 $ \\
    MFCN-wavelet-spectral-1.0 & $ 0.76 \pm 0.14 $ & $ 0.65 \pm 0.32 $ & $ 56.1264 \pm 2.1844 $ & $ 38 \pm 28 $ & $ 35.43 \pm 27.36 $ \\
    MFCN-wavelet-approx-dyadic & $ 0.76 \pm 0.17 $ & $ 0.61 \pm 0.37 $ & $ 0.7534 \pm 0.0126 $ & $ 62 \pm 39 $ & $ 0.78 \pm 0.49 $ \\
    GAT & $ 0.74 \pm 0.15 $ & $ 0.53 \pm 0.37 $ & $ 0.1017 \pm 0.0105 $ & $ 51 \pm 36 $ & $ 0.09 \pm 0.06 $ \\
    MFCN-low-pass-approx & $ 0.71 \pm 0.15 $ & $ 0.51 \pm 0.37 $ & $ 0.0920 \pm 0.0160 $ & $ 49 \pm 36 $ & $ 0.08 \pm 0.06 $ \\
    MFCN-wavelet-spectral-0.5 & $ 0.70 \pm 0.18 $ & $ 0.52 \pm 0.38 $ & $ 51.3448 \pm 0.5228 $ & $ 26 \pm 23 $ & $ 22.12 \pm 20.44 $ \\
    GCN & $ 0.70 \pm 0.15 $ & $ 0.47 \pm 0.37 $ & $ 0.0927 \pm 0.0136 $ & $ 39 \pm 29 $ & $ 0.06 \pm 0.05 $ \\
    GraphSAGE & $ 0.68 \pm 0.13 $ & $ 0.32 \pm 0.33 $ & $ 0.0635 \pm 0.0107 $ & $ 29 \pm 29 $ & $ 0.03 \pm 0.03 $ \\
    MFCN-low-pass-spectral-0.5 & $ 0.63 \pm 0.10 $ & $ 0.24 \pm 0.34 $ & $ 50.1821 \pm 0.6820 $ & $ 6 \pm 7 $ & $ 4.69 \pm 5.70 $ \\
    MFCN-low-pass-spectral-1.0 & $ 0.57 \pm 0.12 $ & $ 0.00 \pm 0.00 $ & $ 50.3428 \pm 1.5880 $ & $ 3 \pm 2 $ & $ 2.27 \pm 1.39 $ \\
    \botrule
    \end{tabular}
\end{table}

\subsubsection{Melanoma patient manifold classification task}\label{sec: melanoma_results}

We apply various MFCN models to a high-dimensional biomedical dataset. Following the lead of \citet{chew2022geometric}, we consider a subset of data originally collected by \citet{ptacek202152}, consisting of 29 protein expression levels measured from hundreds of samples of T-lymphocyte cells from each of $54$ patients with melanoma who received checkpoint blockade immunotherapy. Cell sample sizes range from 489 to 1784 per patient.  

The learning task is to classify whether the origin patient showed a response to immunotherapy, formulated as a binary classification problem (where the two classes consists of patients who showed a partial or complete response, versus patients who experienced stable disease or remission). Mathematically, we model the cells from each patient as an eight-dimensional manifold lying in 29-dimensional space and treat this task as a manifold classification problem where the number of sample points per manifold ranges from 489 to 1784. For simplicity, we use the coordinates of each point (in $\mathbb{R}^{29}$) as the input signals.

Due to the small size of the dataset, we perform five repetitions of $10$-fold cross-validation, dividing the $54$ manifolds into $10$ folds using a stratified-sampling strategy. That is, to better ensure representative training and validation sets, we stratify each fold sample along two strata so that each fold contains roughly equal proportions of (1) responder and non-responder patient samples (i.e., balanced classes), and (2) Roman numeral cancer stages at diagnosis (i.e., I, II, III, or IV).  Additionally, due to the small size of the dataset, we omit the Combine-Features step in order to reduce the number of trainable parameters. We summarize model performance metrics across repetitions by calculating each model's mean of means (i.e., each mean within this `grand' mean is from a complete 10-fold cross-validation procedure in which every sample has been part of a validation set exactly once). We compute standard deviations by taking the variance of the means within each run of cross validation;  then we average this variance over each of the five repetitions and take the square root.

On this data set, when constructing the wavelets used in the MFCN-wavelet-approx model, in addition to the more standard dyadic wavelets $\mathbf{P}_n^{2^{j-1}}-\mathbf{P}_n^{2^{j}}$, we also consider wavelets of the form $\mathbf{P}_n^{t_{j-1}}-\mathbf{P}_n^{t_j}$, where $t_0<t_1<\ldots<t_J$ is a general increasing sequence of diffusion scales. We note that  generalized wavelets such as these were previously considered in \citet{tong2022learnable}, where the diffusion scales $t_j$ were chosen via a differentiable selection matrix. However, here, we introduce a new method, which we name \emph{Infogain}, that selects the diffusion scales $t_j$ on a channel-by-channel basis using ideas from information theory aiming to ensure that each wavelet feature contains (approximately) equal amounts of information. 

This method presents a data-derived strategy for selecting the diffusion scales without the need for large datasets or labeled training data.
This is in contrast to the method used introduced in 
\citet{tong2022learnable} (and discussed in Section \ref{ex: learnable scattering}) which learned the scales through the back-propagation of a differentiable selection matrix (and therefore requires labeled training data). Infogain may thus be viewed as a lightweight alternative to that method, which can also be applied to small or unlabeled data sets. Additionally, it differs from both \citet{tong2022learnable}, as well as more traditional dyadic wavelets, in that it is able select different scales for each channel. For simplicity and computational efficiency, we only utilize the Infogain method in the first MFCN layer (i.e., for initial hidden feature extraction), and then subsequently use dyadic scales in the later layers. Further details on Infogain are provided in Section \ref{sec: info}.

The results in Table \ref{table:melanoma_results} suggest that utilizing Infogain
improves the effectiveness of the corresponding MFCN for the difficult-to-classify, high-dimensional manifolds in the melanoma patient dataset, increasing the mean classification accuracy by five percentage points (relative to MFCN-wavelet-approx-dyadic). More generally, we see that most of the  wavelet-based MFCN models perform well, whereas most of the low-pass variations and most of the message passing networks perform less well (although interestingly we see that MFCN-wavelet-spectral-0.5 and MFCN-low-pass-approx achieve similar results). This suggests that the wavelet-based MFCNs may be better suited to some manifold classification tasks, perhaps due to their ability to capture multi-scale geometric structure. Notably, GIN is the best performing baseline model, and the second best model overall, whereas GraphSage is the worst performing baseline. This is a direct reversal of the regression tasks considered in Section \ref{sec: ellipsoids_results} where GIN failed to learn and GraphSAGE was the best performing baseline.

\begin{figure*}[hbt!]
    \begin{center}
    \includegraphics[width=0.8\textwidth, trim={0 0 0 1.4cm},clip]{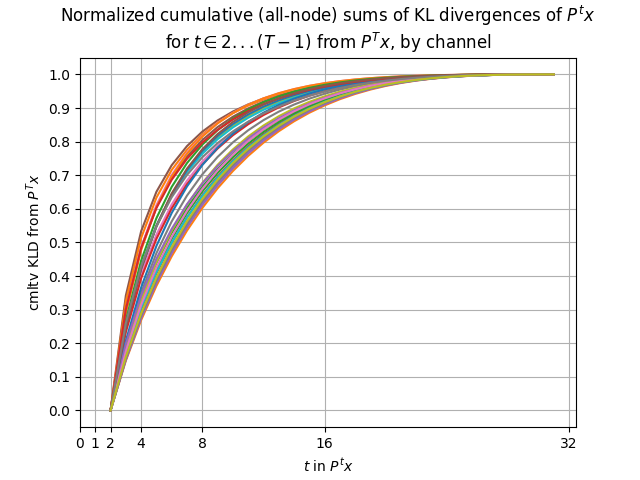}
    \caption{Plot of relative information gain by diffusion time step $t$ and channel, for the 29 channels/features (protein expression levels) in the melanoma patient manifold dataset. The different rates of relative information gain across diffusion steps are reflected in the varying steepness of the channel arcs. For reference, the $t$ axis is ticked in the dyadic scale.}
    \label{fig:melanoma_P_Infogain_cmltv_klds_plot}
    \end{center}
\end{figure*}

\subsubsection{Infogain}\label{sec: info}

In this section, we provide more details on the Infogain procedure, which we use to select diffusion scales for the wavelets on the melanoma data set in the MFCN-wavelet-approx model. As noted in Section \ref{sec: melanoma_results}, the main idea behind Infogain is to select different scales for each channel so that equal amounts of information are contained in each wavelet feature. This allows one to construct a set of wavelets which utilizes the most `important' scales without the need for a priori knowledge.

For simplicity, for each channel (feature) $c$, we set the first several scales by $t^c_0=0, t^c_1=1$, and $t^c_2=2$ so that first two wavelets are always given by $w_{0}(\mathbf{P}_n) = \mathbf{I} - \mathbf{P}_n$ and $w_{1}(\mathbf{P}_n) = \mathbf{P}_n- \mathbf{P}_n^{2}$ (the same as in the dyadic case). We then select the remaining scales $t^c_3,\ldots,t^c_J$ based on the objective of maintaining roughly equal intervals of relative information gain, also known as relative entropy or Kullback-Leibler (KL) divergence, within each filter and channel. More specifically, we fix $J\in\mathbb{N}$ and a maximal diffusion scale $t_J$. Then, for each graph, we consider the $n\times C$ feature matrix $\mathbf{X}$, where $C$ is the number of channels, and:
\begin{enumerate}
    \item Compute $\mathbf{P}_n^{t}\mathbf{x}_c$ for each $1\leq c\leq C$ and each $t = 2,\ldots,t_J$ (recursively via sparse matrix-vector multiplication).
    \item Normalize each $\mathbf{P}_n^{t} \mathbf{x}_{c}$ 
    into probability vectors $\mathbf{q}^{t}_{c}$ for $t = 2,\ldots,t_J$ by first applying min-max scaling to map the entries of each vector into $[0,1]$ and then applying $\ell^1$-normalization.
    \item Replace the zeros in each $\mathbf{q}^{t}_{c}$, with a small nonzero value to avoid arbitrarily small values from skewing information calculations. (Here, we use $\frac{1}{2}\mathrm{min}\{\mathbf{q}^{t}_{c}: \mathbf{q}^{t}_{c} > 0\}$.) 
    \item Compute the KL divergences ($D_{KL}$) of between $\mathbf{q}_c^t$ and $\mathbf{q}_c^{t_J}$ for $t=2,\ldots,t_{J-1}$, i.e.,
    $$(D_{KL})_{t,c}=D_{KL}\big(\mathbf{q}^{t}_{c} \; \| \; \mathbf{q}^{t_J}_{c}\big) = \sum_{k=1}^{n} \mathbf{q}^{t}_{c}(k) \; \mathrm{log}\bigg(\frac{\mathbf{q}^{t}_{c}(k)}{\mathbf{q}^{t_J}_{c}(k)} \bigg).$$
\end{enumerate}

Next, for each channel $c$, $1\leq c\leq C$, we then do the following:
\begin{enumerate}
    \item For each time $t$, sum the $(D_{KL})_{t,c}$ over all graphs: i.e., $(D_{KL})_{t,c}^{\text{total}}=\sum_{i=1} (D_{KL})_{t,c}(i)$, where $(D_{KL})_{t,c}(i)$ is the KL divergence at time $t$ and channel $c$ on the $i$-th graph. (Here, one can optionally re-weight these sums for unbalanced classes, if, for example, the downstream task is manifold classification.)
    \item For each $t=2,\ldots,t_{J-1}$ compute the cumulative sums $S_{t,c}=\sum_{s=2}^t(D_{KL})^{\text{total}}_{t,c}$.
    \item Apply min-max rescaling to the each of the $\{S_{t,c}\}_{t=2}^{t_{J-1}}$ in order for it to have minimal value $0$ and maximal value $1$. 
    \item Finally, select the diffusion scales $t^c_j$ (i.e., the powers to which $\mathbf{P}_n$ is raised)  so that the $S_{t^c_j,c}$ are as evenly spaced as possible given the desired number of scales.
\end{enumerate}

After this procedure, we then define $w_{j_c}(\mathbf{P}_n)=\mathbf{P}_n^{t_{j-1}^c}-\mathbf{P}_n^{t_{j}^c}$ for $2\leq j\leq J$ and define $a_J(\mathbf{P}_n)=\mathbf{P}_n^{t_J}$.
Importantly, we emphasize that these scales are learned independently for each channel $c$ so that we may learn different scales $t_j^c$ for each channel. For example, in our experiments on the melanoma data set, Infogain chooses scales $\{1, 2, 3, 4, 5, 6, 7, 9, 13, 32\}$ for channel 8 (protein `CD31') and $\{1, 2, 3, 4, 6, 7, 9, 11, 15, 32\}$ for channel 2 (protein `CD11b'). 
The differing scales in channel are due to the fact that renormalized cumulative sums $\{S_{t,c}\}_{t=2}^{t_{J-1}}$ are different for each channel as illustrated in Figure \ref{fig:melanoma_P_Infogain_cmltv_klds_plot} (for Infogain applied to the melanoma data set).
Therefore, we choose different scales for each channel which allows us to ensure that there are approximately equal amounts of information in each of the wavelet features $w_{j_c}(\mathbf{P}_n)\mathbf{x}_c$. 

We note that, in general, it is not possible to choose integer scales $t_j^c$ so that the relative entropies are exactly evenly spaced. As a practical strategy, one can designate `selector quantiles' and pick  $t_j^c$ to be the first integer where $S_{t_{j}^c,c}$ exceeds a given quantile.
For example, if one chooses selector quantiles of $[0.2, 0.4, 0.6, 0.8]$, for a channel $c$, then $t^c_{3}$ will be diffusion scale where the channel's normalized cumulative information gain, $S_{t_{j}^c,c}$ exceeds $0.2$. We also note the Infogain procedure is parallelizable as the first four steps can be parallelized over the graphs and the next four can be parallelized over channels.

\section{Conclusions}\label{sec: conclusion}

We have introduced a new framework for analyzing and implementing manifold neural networks that we call manifold filter-combine networks. This framework naturally allows us to think about many interesting classes of MNNs such as the manifold analogs of GCNs and several relaxed variations of the manifold scattering transform. Additionally, we have provided methods for implementing such networks when one does not have global knowledge of the manifold, but merely has access to $n$ sample points, that  provably converge to their continuum limit as $n\rightarrow\infty$.

\backmatter

\bmhead{Acknowledgments}

The authors thank Jeff Calder, Anna Little, and Luana Ruiz for helpful discussion that greatly improved the quality of our exposition. DJ, MP and SK were funded in part by NSF-DMS-2327211. MP was funded in part by NSF-OIA-2242769. SK was funded in part by NSF Career Grant-2047856. DN was funded in part by NSF-DMS-2408912. JC was funded in part by NSF-DMS-2136090 and NSF-DGE-2034835.

\section*{Declarations}

\subsection*{Competing interests}
On behalf of all authors, the corresponding author states that there is no conflict of interest.

\subsection*{Code availability}
Code to reproduce our experiments can be found at \url{https://github.com/dj408/mfcn}. 

\subsection*{Data availability}
The synthetic data for the convergence and ellipsoids regression experiments may be reproduced from the available code. Data for the melanoma patient manifold classification experiment is available on Mendeley Data at \url{https://data.mendeley.com/datasets/79y7bht7tf/}.

\bibliography{sn-bibliography}

\begin{appendices}

\section{The Proof of Lemma \ref{lem: bernstein}}\label{sec: proof of bernstein}

\begin{proof}

Define random variables $\{X_i\}_{i=1}^n$ by $X_i \coloneqq f(x_i)g(x_i)$ and note that by definition we have 
\[\langle P_n f, P_n g \rangle_2 = \frac{1}{n} \sum_{i = 1}^n f(x_i)g(x_i) = \frac{1}{n}\sum_{i=1}^n X_i.\]
 Since the $x_i$ are sampled i.i.d.~with density $\rho$, we have 
\[\mathbb{E}[X_i] = \int_\mathcal{M}f(x)g(x)\rho(x)dx=\langle f, g \rangle_{L^2(\mathcal{M})}.\]
Therefore, letting $\sigma^2 \coloneqq \mathbb{E}[X_i^2] - \mathbb{E}[X_i]^2$ and $M \coloneqq \|fg - \mathbb{E}[X_i]\|_{L^\infty(\mathcal{M})}$, we see that by Bernstein's inequality, we have
\begin{align*}
    \mathbb{P}(|\langle P_n f, P_n g \rangle_2 - \langle f, g \rangle_{L^2(\mathcal{M})}| > \eta ) &= \mathbb{P}\left(\left | \frac{1}{n}\sum_{i=1}^n X_i - \frac{1}{n} \sum_{i=1}^n \mathbb{E}[X_i]\right | > \eta \right) \\
    &= \mathbb{P}\left(\left | \sum_{i=1}^n X_i - \sum_{i=1}^n \mathbb{E}[X_i]\right | > n\eta \right)\\
    &\leq 2 \exp \left (- \frac{\frac{1}{2}n\eta^2}{\sigma^2 + \frac{1}{3}M\eta}\right ).
\end{align*}
Setting $\eta = 6 \sqrt{\dfrac{\sigma^2 \log(n)}{n}}$, we see that for $n$ large enough so that $1 + \dfrac{M\eta}{3\sigma^2} < 2$, we have
\begin{align*}
    \mathbb{P}(|\langle P_n f, P_n g \rangle_2 - \langle f, g \rangle_{L^2(\mathcal{M})}| > \eta ) &\leq 2 \exp \left (- \frac{18\sigma^2\log(n)}{\sigma^2 + \frac{1}{3}M\eta}\right ) \\
    &= 2 \exp \left (- \frac{18\log(n)}{1 + \frac{M\eta}{3\sigma^2}}\right ) \\
    &< 2 \exp \left ( -9\log(n)\right ) \\
    &= \frac{2}{n^9}.
\end{align*}
We note that $\sigma^2\leq \mathbb{E}[X_i^2]=\langle f^2,g^2\rangle_{L^2(\mathcal{M})}$.
Therefore, by the Cauchy-Schwarz inequality, with probability at least $1 - \frac{2}{n^9}$, we have
\begin{align*}
    |\langle P_n f, P_n g \rangle_2 - \langle f, g \rangle_{L^2(\mathcal{M})}| &\leq 6  \sqrt{\frac{\sigma^2\log(n)}{n}} \\
    &\leq 6 \sqrt{\frac{\log(n)}{n}} \sqrt{\langle f^2,g^2\rangle_{L^2(\mathcal{M})}} \\
    &\leq 6 \sqrt{\frac{\log(n)}{n}} \|f\|_{L^4(\mathcal{M})} \|g\|_{L^4(\mathcal{M})}.
\end{align*}

\end{proof}

\section{The proof of Theorem \ref{thm: Filter Error short}}\label{sec: proof of theorem filter error}
We will prove Theorem \ref{thm: Filter Error} stated below. Theorem \ref{thm: Filter Error short} will then follow immediately by combining Theorem \ref{thm: Filter Error} with Theorem \ref{thm: recall Calder results} with Lemma  \ref{lem: bernstein}. Importantly, we note that our result is given in terms of the quantites $\alpha_n$, $\beta_n$, and $\gamma_n$ which quantify the errors in the eigenvalues, the eigenvectors, and the Fourier coefficients. Therefore, if future works were to produce improved versions of Theorems \ref{thm: recall Calder results} or Lemma \ref{lem: bernstein}, one would readily obtain an improved version of Theorem \ref{thm: Filter Error short}. 

\begin{theorem}\label{thm: Filter Error}
 Let $w:[0,\infty)\rightarrow \mathbb{R}$, $\|w\|_{L^\infty([0,\infty))}\leq 1$,
let $f\in L^2(\mathcal{M})$ be a continuous function,  and assume that either $f$ or $w(\mathcal{L})$ is $\kappa$-bandlimited. 
Assume that there exist sequences of real numbers $\{\alpha_n\}_{n=1}^\infty$, $\{\beta_n\}_{n=1}^\infty$, $\{\gamma_n\}_{n=1}^\infty$, with  $\lim_{n\rightarrow\infty} \alpha_n=\lim_{n\rightarrow\infty}\beta_n=\lim_{n\rightarrow\infty}\gamma_n=0,$ such that for all $1\leq i \leq \kappa$ and for $n$ sufficiently large, we have 
\begin{equation}\label{eqn: alphabetagamma}
    |\lambda_i-\lambda^{n}_i|\leq \alpha_n,\quad    \|P_n\phi_i-\phi_i^n\|_2\leq\beta_n,\quad
    |\langle  P_nf, P_ng \rangle_2 - \langle f,g\rangle_{L^2(\mathcal{M})}| \leq  \gamma_n^2\|f\|_{L^4(\mathcal{M})}\|g\|_{L^4(\mathcal{M})},
\end{equation}
Then for $n$ large enough such that \eqref{eqn: alphabetagamma} holds and $\alpha_nA_{\text{Lip}}(w),\beta_n,\gamma_n\|\phi_i\|_{L^4(\mathcal{M})}\leq 1$ for all $i\leq \kappa$, we have 
\begin{equation}\label{eqn: filter stability no x}
\|w(\mathbf{L}_n)P_nf-P_nw(\mathcal{L})f\|_2\leq  
6\kappa\left((A_{\text{Lip}}(w)\alpha_n+\beta_n)\|f\|_{L^2(\mathcal{M})}+\gamma_n\|f\|_{L^4(\mathcal{M})}\right).
\end{equation}
Furthermore, under the same assumptions on $n$,  we have 
\begin{equation}\label{eqn: filter stability with x}
\|w(\mathbf{L}_n)\mathbf{x}-P_nw(\mathcal{L})f\|_2\leq \|\mathbf{x}-P_nf\|_2 + 
6\kappa\left((A_{\text{Lip}}(w)\alpha_n+\beta_n)\|f\|_{L^2(\mathcal{M})}+\gamma_n\|f\|_{L^4(\mathcal{M})}\right)
\end{equation}
for all $\mathbf{x}\in\mathbb{R}^n$. 
In particular, if $\mathbf{x}=P_nf$, and $A_{\text{Lip}}(w)$ is finite, then \eqref{eqn: filter stability with x} implies that  
\begin{equation*}
\lim_{n\rightarrow\infty}    \|w(\mathbf{L}_n)\mathbf{x}-P_nw(\mathcal{L})f\|_2=0.
\end{equation*}

\end{theorem}

\begin{proof}[The proof of Theorem \ref{thm: Filter Error}]

We first note that if either $w$ or $f$ is $\kappa$-bandlimited, we have
\begin{align}
&\|w(\mathbf{L}_n)P_nf-P_nw(\mathcal{L})f\|_2\nonumber\\
= &\left\| \sum_{i=1}^\kappa w(\lambda_i^n)\langle P_nf,\phi_i^n\rangle_2\phi_i^n - \sum_{i=1}^\kappa w(\lambda_i)\langle f,\phi_i\rangle_\mathcal{M}P_n\phi_i\right\|_2\nonumber\\
\leq&\left\| \sum_{i=1}^\kappa (w(\lambda_i^n)-w(\lambda_i))\langle P_nf,\phi_i^n\rangle_2\phi_i^n\right\|_2+\left\|\sum_{i=1}^\kappa w(\lambda_i)(\langle P_nf,\phi_i^n\rangle_2\phi_i^n- \langle f,\phi_i\rangle_\mathcal{M}P_n\phi_i)\right\|_2.\label{eqn: different eigenvectorsv2}
\end{align}

To bound the first term from \eqref{eqn: different eigenvectorsv2}, we note that by the triangle inequality, the Cauchy-Schwarz inequality, and the assumption that $n$ is large enough so that $A_{\text{Lip}}(w)\alpha_n\leq 1$, we have 
\begin{align}
    \left\| \sum_{i=1}^\kappa (w(\lambda_i^n) - w(\lambda_i))\langle P_nf,\phi_i^n\rangle_2\phi_i^n\right\|_2
    \leq& \max_{1\leq i \leq \kappa} |w(\lambda_i^n)- w(\lambda_i)| \sum_{i=1}^\kappa \|P_n f\|_2 \|\phi_i^n\|^2_2\nonumber\\
    \leq& A_{\text{Lip}}(w)\alpha_n \sum_{i=1}^\kappa \|P_n f\|_2 \|\phi_i^n\|^2_2\nonumber\\
    \leq& A_{\text{Lip}}(w)\kappa\alpha_n  \|P_n f\|_2\nonumber\\ 
    \leq& \kappa  \left(A_{\text{Lip}}(w)\alpha_n\|f\|_{L^2(\mathcal{M})}+\gamma_n\|f\|_{L^4(\mathcal{M})}\right),\nonumber
\end{align}
where we use the fact that $\|\phi_i^n\|_2^2=1$ and that 
\begin{equation}\label{eqn: pnboundv2}\|P_nf\|_2\leq \left(\|f\|_{L^2(\mathcal{M})}^2 + \gamma_n^2\|f\|_{L^4(\mathcal{M})}^2\right)^{1/2}\leq \|f\|_{L^2(\mathcal{M})}+\gamma_n\|f\|_{L^4(\mathcal{M})}.\end{equation}

Now, turning our attention to the second term from \eqref{eqn: different eigenvectorsv2}, we have 
\begin{align}
&\left\|\sum_{i=1}^\kappa w(\lambda_i)(\langle P_nf,\phi_i^n\rangle_2\phi_i^n- \langle f,\phi_i\rangle_{L^2(\mathcal{M})}P_n\phi_i)\right\|_2\nonumber\\
    \leq& \left\|\sum_{i=1}^\kappa w(\lambda_i)\langle P_nf,\phi_i^n\rangle_2(\phi_i^n-P_n\phi_i)\right\|_2
+\left\|\sum_{i=1}^\kappa w(\lambda_i)(\langle P_nf,\phi_i^n\rangle_2- \langle f,\phi_i\rangle_{L^2(\mathcal{M})}P_n\phi_i\right\|_2\label{eqn: use Hoeffding this timev2}.
\end{align}

By the assumption \eqref{eqn: alphabetagamma}, we have $\|\phi_i^n-P_n\phi_i\|_2\leq \beta_n$.
Therefore, since $w$ is non-amplifying, we see
\begin{align}
\left\|\sum_{i=1}^\kappa w(\lambda_i)\langle P_nf,\phi_i^n\rangle_2(\phi_i^n-P_n\phi_i)\right\|_2
    &\leq \kappa \max_{1\leq i\leq \kappa} |\langle P_nf,\phi_i^n\rangle_2|\|\phi_i^n-P_n\phi_i\|_2\nonumber\\
    &\leq \kappa\beta_n\|P_nf\|_2\nonumber\\
    &\leq \kappa\left (\beta_n\|f\|_{L^2(\mathcal{M})}+\gamma_n\|f\|_{L^4(\mathcal{M})}\right ),\label{eqn: evec diffv2}
\end{align}
where the final inequality follows from \eqref{eqn: pnboundv2} and the assumption that $n$ is large enough so that $\beta_n \leq 1$.
Meanwhile, the second term from \eqref{eqn: use Hoeffding this timev2} can be bounded by 
\begin{align*}
&\left\|\sum_{i=1}^\kappa w(\lambda_i)(\langle P_nf,\phi_i^n\rangle_2- \langle f,\phi_i\rangle_\mathcal{M})P_n\phi_i\right\|_2\\
\leq&\sum_{i=1}^\kappa |w(\lambda_i)| |\langle P_nf,\phi_i^n\rangle_2- \langle f,\phi_i\rangle_\mathcal{M}|\|P_n\phi_i\|_2\\   \leq&\sum_{i=1}^\kappa |\langle P_nf,\phi_i^n\rangle_2- \langle f,\phi_i\rangle_\mathcal{M}|\|P_n\phi_i\|_2\\
\leq&\sum_{i=1}^\kappa |\langle P_nf,\phi_i^n\rangle_2-\langle P_nf,P_n\phi_i\rangle_2|\|P_n\phi_i\|_2+\sum_{i = 1}^\kappa |\langle P_nf,P_n\phi_i\rangle_2- \langle f,\phi_i\rangle_\mathcal{M}|\|P_n\phi_i\|_2.
\end{align*}
 By the Cauchy-Schwarz inequality, \eqref{eqn: alphabetagamma},  \eqref{eqn: pnboundv2}, and the assumption that $n$ is large enough so that $\beta_n\leq 1$, we have
\begin{align*}
    |\langle P_nf,\phi_i^n\rangle_2-\langle P_nf,P_n\phi_i\rangle_2| &\leq \| P_nf\|_2 \|\phi_i^n-P_n\phi_i\|_2 \\
    &\leq \beta_n\left(\|f\|_{L^2(\mathcal{M})}+\gamma_n\|f\|_{L^4(\mathcal{M})}\right) \\
    &\leq\left(\beta_n\|f\|_{L^2(\mathcal{M})}+\gamma_n\|f\|_{L^4(\mathcal{M})}\right).
\end{align*}
And also by \eqref{eqn: alphabetagamma} and \eqref{eqn: pnboundv2} we have 
\begin{align*}
    |\langle P_nf,P_n\phi_i\rangle_2- \langle f,\phi_i\rangle_2| \leq \gamma_n^2\|f\|_{L^4(\mathcal{M})}\|\phi_i\|_{L^4(\mathcal{M})},\quad \text{and} \quad \|P_n\phi_i\|_2\leq 1+\gamma_n\|\phi_i\|_{L^4(\mathcal{M})}.\end{align*}
Since $\kappa$ is fixed and  $\lim_{n\rightarrow\infty}\gamma_n=0,$ for sufficiently large $n$ we have $\gamma_n\|\phi_i\|_{L^4(\mathcal{M})}\leq 1$ for all $i\leq \kappa.$ This implies 
\begin{align*}
    |\langle P_nf,P_n\phi_i\rangle_2- \langle f,\phi_i\rangle_2| \leq \gamma_n\|f\|_{L^4(\mathcal{M})},\quad\text{and}\quad
\|P_n\phi_i\|_2\leq 1+\gamma_n\|\phi_i\|_{L^4(\mathcal{M})}\leq 2.
\end{align*}

Therefore, the second term from \eqref{eqn: use Hoeffding this timev2} can be bounded by 
\begin{align}
    &\left\|\sum_{i=1}^\kappa w(\lambda_i)(\langle P_nf,\phi_i^n\rangle_2 -\langle f,\phi_i\rangle_\mathcal{M})P_n\phi_i\right\|_2 \nonumber\\
    \leq&\sum_{i=1}^\kappa |\langle P_nf,\phi_i^n\rangle_2-\langle P_nf,P_n\phi_i\rangle_2|\|P_n\phi_i\|_2 \nonumber +\sum_{i=1}^\kappa|\langle P_nf,P_n\phi_i\rangle_2- \langle f,\phi_i\rangle_2|\|P_n\phi_i\|_2\nonumber\\
\leq&\sum_{i=1}^\kappa\left(\beta_n\|f\|_{L^2(\mathcal{M})}+\gamma_n\|f\|_{L^4(\mathcal{M})}\right)\|P_n\phi_i\|_2 +\sum_{i=1}^\kappa  \gamma_n\|f\|_{L^4(\mathcal{M})}\|P_n\phi_i\|_2\nonumber\\
\leq&4\kappa\left( \beta_n\|f\|_{L^2(\mathcal{M})}+\gamma_n\|f\|_{L^4(\mathcal{M})}\right).\label{eqn: inner pdt diffv2}
\end{align}

Therefore, combining Equations \eqref{eqn: different eigenvectorsv2} through \eqref{eqn: inner pdt diffv2} yields 
\begin{align*}
&\|w(\mathbf{L}_n)P_nf-P_nw(\mathcal{L})f\|_2\nonumber\\
\leq &\left\| \sum_{i=1}^\kappa (w(\lambda_i^n) - w(\lambda_i))\langle P_nf,\phi_i^n\rangle_2\phi_i^n\right\|_2+\left\|\sum_{i=1}^\kappa w(\lambda_i)(\langle P_nf,\phi_i^n\rangle_2\phi_i^n-\langle f,\phi_i\rangle_\mathcal{M}P_n\phi_i)\right\|_2\\
\leq& \kappa  (A_{\text{Lip}}(w)\alpha_n\|f\|_{L^2(\mathcal{M})}+\gamma_n\|f\|_{L^4(\mathcal{M})})+ 5\kappa \left(\beta_n\|f\|_{L^2(\mathcal{M})}+\gamma_n\|f\|_{L^4(\mathcal{M})}\right)\\
\leq &
6\kappa\left((A_{\text{Lip}(w)}\alpha_n+\beta_n)\|f\|_{L^2(\mathcal{M})}+\gamma_n\|f\|_{L^4(\mathcal{M})}\right)
\end{align*}
thus completing the proof of \eqref{eqn: filter stability no x}.

To prove \eqref{eqn: filter stability with x}, we observe that since $\|w\|_{L^\infty([0,\infty))} \leq 1,$ we have \begin{align}
\|w(\mathbf{L}_n)\mathbf{x}-w(\mathbf{L}_n)P_nf\|_2 &=\|w(\mathbf{L}_n)(\mathbf{x}-P_nf)\|_2 \nonumber\\&=\|\sum_{i=1}^n w(\lambda_i^n) \langle \mathbf{x}-P_nf, \phi_i^n\rangle_2\phi_i^n\|_2\\
&=\left(\sum_{i=1}^n |w(\lambda_i^n)|^2|\langle \mathbf{x}-P_nf, \phi_i^n\rangle_2|^2\right)^{1/2}\nonumber\\
&\leq \left(\sum_{i=1}^n |\langle \mathbf{x}-P_nf, \phi_i^n\rangle_2|^2\right)^{1/2}\nonumber\\
&=\|\mathbf{x}-P_nf\|_2.\label{eqn: easy recursion termv2}
\end{align}
Therefore, by the triangle inequality, we have \begin{align*}
\|w(\mathbf{L}_n)\mathbf{x}-P_nw(\mathcal{L})f\|_2
\leq& 
\|w(\mathbf{L}_n)\mathbf{x}-w(\mathbf{L}_n)P_nf\|_2 +
\|w(\mathbf{L}_n)P_nf-P_nw(\mathcal{L})f\|_2\\
\leq& \|\mathbf{x}-P_nf\|_2 + 
6\kappa\left((A_{\text{Lip}(w)}\alpha_n+\beta_n)\|f\|_{L^2(\mathcal{M})}+\gamma_n\|f\|_{L^4(\mathcal{M})}\right)\end{align*}
as desired.
\end{proof}

\section{The Proof of Theorem \ref{thm: bound given filter bound short}} \label{sec: proof of network theorem}

Below, as alluded to in Section \ref{sec: convergence},  we will prove a more general result, Theorem \ref{thm: bound given filter bound}. Theorem \ref{thm: bound given filter bound short} will then follow as a special case. Specifically, we will consider \emph{generalized MFCNs} which are defined in exactly the same manner as in Section  \ref{sec: MFCN} with the exception that in the Filter step, we allow arbitrary bounded linear operators on $\mathcal{L}^2(\mathcal{M})$, denoted by  $W_{j,k}^{(\ell)}$ rather than restricting attention to spectral convolutional operators $w_{j,k}^{(\ell)}(\mathcal{L})$. In our convergence analysis, we will let, $\mathbf{W}^{(\ell)}_{j,k}$ denote a matrix used as a discrete approximation of $W_{j,k}^{(\ell)}$ in the Filter step of the numerical implementation. After we prove Theorem \ref{thm: bound given filter bound}, Theorem \ref{thm: bound given filter bound short} will then follow immediately by observing that Theorem \ref{thm: Filter Error} allows us to choose 
$$
\epsilon_{\ell,n}=\mathcal{O}\left(\left( \frac{\log(n)}{n}\right)^{\frac{1}{d+4}}\right)(A_{\text{Lip}}(w)+1)\|f\|_{L^2(\mathcal{M})}+\mathcal{O}\left(\left(\frac{\log(n)}{n}\right)^{1/4}\right)\|f\|_{\mathbf{L}^4(\mathcal{M})}
$$
in the case where the operators $W_{j,k}^{(\ell)}$ are spectral filters satisfying the assumptions of Theorem \ref{thm: bound given filter bound short}.

We note that the generality of Theorem \ref{thm: bound given filter bound} is motivated in part by the observation that utilizing spectral filters directly with an explicit eigendecomposition has a significant computational cost (i.e., $\mathcal{O}(\kappa n^2)$) and in practice one would likely choose to approximate spectral filters in more computational efficient ways. In this paper, we do not propose a provably accurate method for such an approximation.  We do, however, state our theorem below in generality based on this observation so that if future work provides rigorous convergence guarantees for a such a method, one will immediately obtain a convergence result for the associated generalized MFCNs. Similarly, convergence results for individual spectral filters  applied to other class of functions would also immediately lead to such convergence results for the associated generalized MFCN that relate the rate of convergence to the weights of the network.

\begin{theorem}\label{thm: bound given filter bound}
Let $f \in \mathcal{C}(\mathcal{M})$, and suppose that for all $\ell$, there exists $\epsilon_{\ell,n}>0$ such that we have \begin{equation*}
\|P_nW_{j,k}^{(\ell)}f_k^{(\ell)}-\mathbf{W}^{(\ell)}_{j,k}\mathbf{x}^{(\ell)}_k\|_2 \leq \|\mathbf{x}^{(\ell)}_k-P_nf_k^{\ell}\|_2+ \epsilon_{\ell,n}
\end{equation*}
for all $1\leq k \leq C_\ell$.
Let $A_1^{(\ell)}=\max_{j,k}(|\sum_{i=1}^{C_{\ell}} |\theta_{i,k}^{(\ell,j)}|),$ $A_2^{(\ell)}=\max_{j,k}(\sum_{i=1}^{J_{\ell}} |\alpha_{j,i}^{(\ell,k)}|)$ and assume that $\sigma$ is non-expansive, i.e. $|\sigma(x)-\sigma(y)|\leq |x-y|$.
Then, 
$$\|\mathbf{x}_k^{\ell}-P_nf_k^{\ell}\|_2\leq \sum_{i=0}^{\ell-1} \prod_{j=i}^{\ell-1} A_{1}^{(j)} A_{2}^{(j)} \epsilon_{i,n}.$$
\end{theorem}

In order to prove Theorem \ref{thm: bound given filter bound}, we need the following lemma which bounds the error in each step. 

\begin{lemma}\label{lem: other step error}
The errors induced by the non-filtering steps of our network may be bounded by 
\begin{align}
   \| \mathbf{y}_{j,k}^{(\ell)}-P_ng_{j,k}^{(\ell)}\|_2
&\leq \max_{1\leq i\leq C_{\ell}}\|\tilde{\mathbf{x}}^{(\ell)}_{j,k}-P_n\tilde{f}^{(\ell)}_{j,k}\|_2\sum_{i=1}^{C_{\ell}} |\theta_{i,k}^{(\ell,j)}|,\label{eqn: theta error}
\\
\| \tilde{\mathbf{y}}_{j,k}^{(\ell)}-P_n\tilde{g}_{j,k}^{(\ell)}\|_2
&\leq \max_{1\leq i\leq J_{\ell}}\|\mathbf{y}^{(\ell)}_{j,k}-P_n g^{(\ell)}_{j,k}\|_2\sum_{i=1}^{J_{\ell}} |\alpha_{j,i}^{(\ell,k)}|.\label{eqn: alpha error}\\
\|\mathbf{z}^{(\ell)}_{j,k}-P_nh^{(\ell)}_{j,k}\|_2&\leq\|\tilde{\mathbf{y}}^{(\ell)}_{j,k}-P_n\tilde{g}^{(\ell)}_{j,k}\|_2.\label{eqn: sigma error}    
\end{align}
\end{lemma}
\begin{proof}
To verify \eqref{eqn: theta error}, we observe that 
    \begin{align*}
   \| \mathbf{y}_{j,k}^{(\ell)}-P_ng_{j,k}^{(\ell)}\|_2
&=\left\|\sum_{i=1}^{C_{\ell}}\tilde{\mathbf{x}}^{(\ell)}_{j,k}\theta_{i,k}^{(\ell,j)}-P_n\tilde{f}^{(\ell)}_{j,k}\theta_{i,k}^{(\ell,j)}\right\|_2\\
&\leq \sum_{i=1}^{C_{\ell}}|\theta_{i,k}^{(\ell,j)}|\|\tilde{\mathbf{x}}^{(\ell)}_{j,k}-P_n\tilde{f}^{(\ell)}_{j,k}\|_2\nonumber\\
&\leq  \max_{1\leq i\leq C_{\ell}}\|\tilde{\mathbf{x}}^{(\ell)}_{j,k}-P_n\tilde{f}^{(\ell)}_{j,k}\|_2\sum_{i=1}^{C_{\ell}} |\theta_{i,k}^{(\ell,j)}|.
    \end{align*}
The proof of \eqref{eqn: alpha error} is identical to the proof of \eqref{eqn: theta error}. For \eqref{eqn: sigma error}, we see that 
since $\sigma$ is non-expansive we have    
\begin{align*}
\|\mathbf{z}^{(\ell)}_{j,k}-P_nh^{(\ell)}_{j,k}\|^2_2
&=\sum_{i=1}^n|
(\mathbf{z}^{(\ell)}_{j,k})(i)-(P_nh^{(\ell)}_{j,k})(i)|^2\\
&=\sum_{i=1}^n|
(\mathbf{z}^{(\ell)}_{j,k})(i)-h^{(\ell)}_{j,k}(x_i)|^2\\
&=\sum_{i=1}^n|
\sigma((\tilde{\mathbf{y}}^{(\ell)}_{j,k})(i))-\sigma(\tilde{g}^{(\ell)}_{j,k}(x_i))|^2\\
&\leq\sum_{i=1}^n|
(\tilde{\mathbf{y}}^{(\ell)}_{j,k})(i)-\tilde{g}^{(\ell)}_{j,k}(x_i)|^2\\
&=\|\tilde{\mathbf{y}}^{(\ell)}_{j,k}-P_n\tilde{g}^{(\ell)}_{j,k}\|^2_2.
\end{align*} 
\end{proof}

\begin{proof}[Proof of Theorem \ref{thm: bound given filter bound}]
It follows from the definition of the reshaping operator that $$\max_k\|\mathbf{x}_k^{(\ell+1)}-P_nf_k^{(\ell+1)}\|_2
= \max_{j,k}\|\mathbf{z}^{(\ell)}_{p,r}-P_nh^{(\ell)}_{p,r}\|_2.$$

Therefore, by Lemma \ref{lem: other step error} we have 
\begin{align*}
\max_k\|\mathbf{x}_k^{(\ell+1)}-P_nf_k^{(\ell+1)}\|_2
=& \max_{j,k}\|\mathbf{z}^{(\ell)}_{p,r}-P_nh^{(\ell)}_{p,r}\|_2.\\
\leq& \|P_n\tilde{g}^{(\ell)}_{j,k}-\tilde{\mathbf{y}}^{(\ell)}_{j,k}\|_2.\\
\leq& A^{(\ell)}_2\max_{j,k}\|P_n g^{(\ell)}_{j,k}-\mathbf{y}^{(\ell)}_{j,k}\|_2\\
\leq& A^{(\ell)}_2A^{(\ell)}_1\max_{j,k}\|P_n \tilde{f}^{(\ell)}_{j,k}-\tilde{\mathbf{x}}^{(\ell)}_{j,k}\|_2\\
\leq& A^{(\ell)}_2A^{(\ell)}_1(\max_k\|\mathbf{x}_k^{(\ell)}-P_nf_k^{(\ell)}\|_2
+\epsilon_{\ell,n})
\end{align*}

Since $\|\mathbf{x}_0^{(\ell)}-P_nf^{(0)}_k\|_2=0$ for all $k$,
we may use induction to conclude that 
\begin{equation*}\|\mathbf{x}^{(\ell)}_{k}-P_nf^{(\ell)}_k\|_2\leq \sum_{i=0}^{\ell-1} \prod_{j=i}^{\ell-1} A_{1}^{(j)} A_{2}^{(j)} \epsilon_{i,n}.
\end{equation*}
\end{proof}

\section{Further  experimental results}\label{appendix more experiments}

\subsection{Convergence Results}\label{appendix: convergence}

Here, we conduct convergence experiments similar to Section \ref{sec: convergence_results}, but in case where $G_n$ is constructed as an $\epsilon$-graph (as per Definition \ref{ex: eps}) as opposed to a $k$-NN graph. In this case,  the manifold Laplacian $\mathcal{L}$ is given by $\mathcal{L}f=-\frac{1}{2\rho}\text{div}(\rho^2\nabla f)$, which reduces to  $-\frac{1}{8\pi}\Delta$ when sampling uniformly on the sphere. 
As before, we sample $n$ points uniformly from $\mathbb{S}^2$, choose $\epsilon$ in accordance with Theorem \ref{thm: recall Calder results}, and construct an $\epsilon$-graph $G_n$. We consider the same function $f=Y^0_1+Y^0_2$  and spectral filter $w(\lambda)=e^{-\lambda}$ as in the previous experiment, and we report the discretization error $\|w(\mathbf{L}_n) P_n f - P_n w(\mathcal{L}) f\|_2$ over 10 trials for $n$ ranging from $2^6$ to $2^{14}$ in Figure \ref{fig:eps_spectral_filter_converg_plot}. Additionally, we plot the convergence of the first two distinct nontrivial eigenvalues in Figure \ref{fig:eps_eigenvals_converg_plot}, and see that the numerical eigenvalues converge to the true values of $\frac{\ell(\ell+1)}{8\pi}$ for $\ell = 1,2$).

\begin{figure*}[hbt!]
    \begin{center}
    \includegraphics[width=1.0\textwidth]{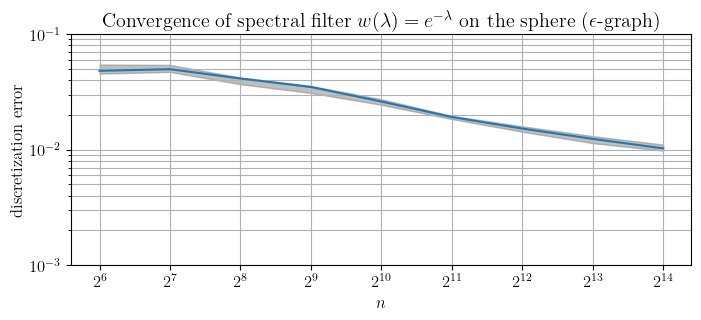}
    \caption{Discretization error for spectral filter $w(\lambda)=e^{-\lambda}$ applied to $f=Y^0_1+Y^0_2$, using an $\epsilon$-graph construction. The median error of 10 runs is shown in blue, against a gray band of the 25th- to 75th-percentile error range.}
    \label{fig:eps_spectral_filter_converg_plot}
    \end{center}
\end{figure*}

\begin{figure*}[hbt!]
    \begin{center}
    \includegraphics[width=1.0\textwidth]{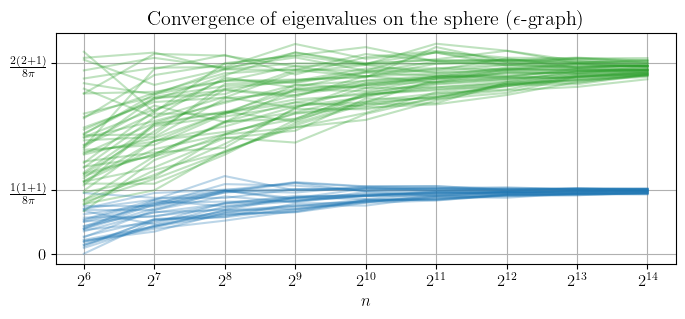}
    \caption{Convergence of first eight, non-zero eigenvalues on the sphere, for an $\epsilon$-graph construction, all 10 runs combined. The blue lines are the first three eigenvalues that converge to the same limit (since the first non-zero eigenvalue of the spherical Laplacian has multiplicity three). Similarly, the green lines are the next five eigenvalues. }
    \label{fig:eps_eigenvals_converg_plot}
    \end{center}
\end{figure*}

\subsection{Regression on Ellipsoids}\label{appendix: ellipsoid}
Here we build upon the regression experiments considered in Section \ref{sec: ellipsoids_results}.
 In order to further understand what type of signals each model is able to learn, we consider the setting where the target function is an individual  eigenfunction, i.e., $f=\phi_i$, $2\leq i \leq 21$. Here we do not apply any noise because our goal is understand the performance of each model to learn within each frequency band. We perform five-fold cross validation for each eigenfunction.
 
 Our results are displayed in Figures \ref{fig:ellipsoids_single_evecs_plot_r2} and \ref{fig:ellipsoids_single_evecs_plot_mse}, where we plot the $R^2$ and MSE scores for the various models as a function of the target eigenfunction. (Results are also displayed in tabular form in Table \ref{tab:ellips_single_all}.) We see that these plots may help us understand which models performed best on the ellipsoid regressions tasks considered in Section \ref{sec: ellipsoids_results} (displayed in Table \ref{table:ellipsoids_combination_signal_results}).  
 Generally, the models which performed best in Section \ref{sec: ellipsoids_results} are able to learn both high- and low-frequency eigenfunctions whereas the other models are unable to learn the higher-frequency eigenfunctions. 
 
 For instance, MFCN-low-pass-spectral-0.5, which was the top performing model in Table \ref{table:ellipsoids_combination_signal_results} achieves strong performance for all $20$ eigenfunctions. The wavelet-based MFCNs, which were the next top methods in Table \ref{table:ellipsoids_combination_signal_results}, are reasonably effective for all of the eigenfunction, although their performance deteriorates a bit after eigenfunction 15. Most of the baseline methods are able to learn the first 11 eigenfunctions, but are unable to learn the higher-frequency eigenfuctions (where the achieve very low $R^2$, if not zero). Among the baselines, the one exception is GraphSAGE, which is able to learn the first 15 eigenfunctions (and achieves near-zero $R^2$ scores after that). Notably, GraphSAGE is also the top performing baseline model in both the noisy and noiseless settings. As in Table \ref{table:ellipsoids_combination_signal_results}, GIN and MFCN-low-pass-spectral-1.0 again fail completely, achieving $R^2$ scores of 0.00 for all of the eigenfunctions.

\begin{figure*}[hbt!]
    \begin{center}
    \includegraphics[width=1.0\textwidth, trim={1.5cm 0 2.0cm 0},clip]{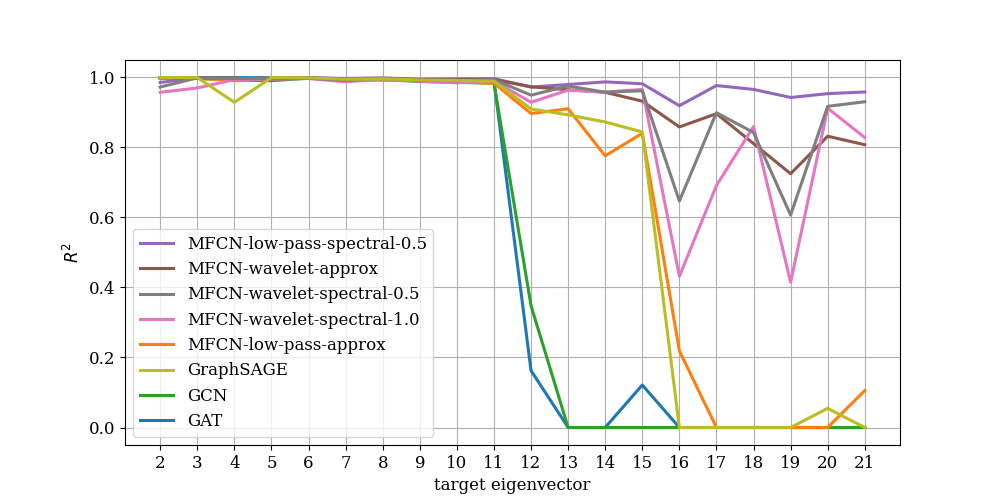}
    \caption{Mean $R^2$ values by target, from five-fold cross-validation for node regression where the targets are the values of a single eigenvector. These target eigenvectors are the first 20 nonconstant eigenvectors of the graph Laplacian of a $k$-NN graph constructed from a uniform sample of 1024 points lying a 2-ellipsoid in 8-d ambient space (i.e. a noiseless sample). Note that MFCN-low-pass-spectral-1.0 and GIN are excluded, as both failed to learn under the experimental regime, consistently returning $R^2$s of zero. Additionally, at eigenfunction 17 and above, the bottom four models also fail to learn, and return $R^2$s near zero.}
    \label{fig:ellipsoids_single_evecs_plot_r2}
    \end{center}
\end{figure*}

\begin{figure*}[hbt!]
    \begin{center}
    \includegraphics[width=1.0\textwidth, trim={1.4cm 0 2.0cm 0},clip]{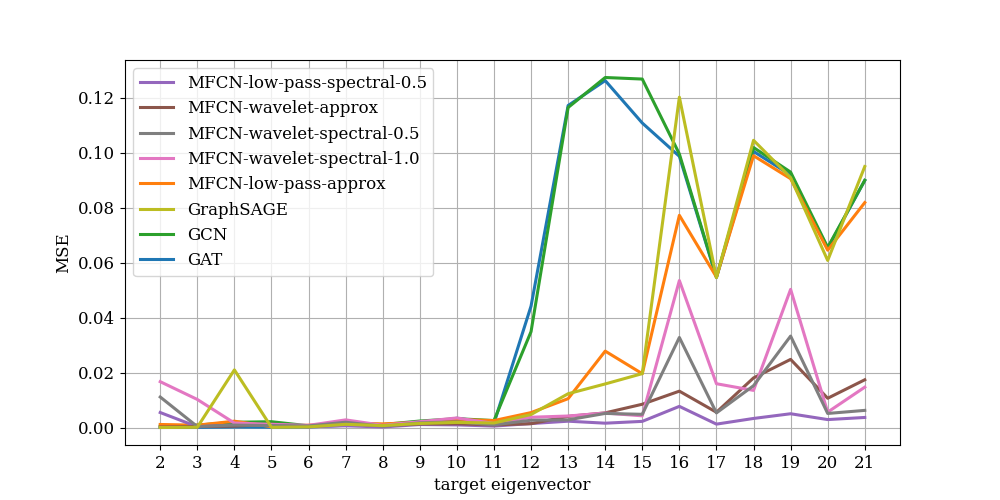}
    \caption{Mean mean squared error (MSE) values by target, from five-fold cross-validation for node regression where the targets are the values of a single eigenvector. These target eigenvectors are the first 20 non-constant eigenvectors of the graph Laplacian of a $k$-NN graph constructed from a uniform sample of 1024 points lying a 2-ellipsoid in 8-d ambient space (i.e. a noiseless sample). Note that MFCN-low-pass-spectral and GIN data are excluded, as both fail to learn under the experimental regime, consistently returning nearly identical MSEs consistent with mean-of-training targets predictions. In addition, at eigenfunction 17 and above, the bottom four models also fail to learn, and return mean-of-training targets predictions.}
    \label{fig:ellipsoids_single_evecs_plot_mse}
    \end{center}
\end{figure*}

\begin{longtable}{lccc}
\caption{Model regression performance metrics on individual eigenfunctions of ellipsoids. Models are sorted within eigenfunction groups by MSE score (by lowest mean MSE, then lowest MSE standard deviation in the case of ties). $R^2$ scores $<0$ (occurring for numerical reasons) have been corrected back to $0$.} \\
\label{tab:ellips_single_all} \\
\hline
\textbf{Model} & \textbf{Eigenfunction} & $\mathbf{R^2}$ & \textbf{MSE} \\\hline
\endfirsthead
\hline
\textbf{Model} & \textbf{Eigenfunction} & $\mathbf{R^2}$ & \textbf{MSE} \\\hline
\endhead
GraphSAGE & 2 & $ 1.0000 \pm 0.0000 $ & $ 0.0000 \pm 0.0000 $ \\
GAT & 2 & $ 1.0000 \pm 0.0000 $ & $ 0.0000 \pm 0.0000 $ \\
MFCN-wavelet-approx & 2 & $ 0.9995 \pm 0.0002 $ & $ 0.0002 \pm 0.0001 $ \\
GCN & 2 & $ 0.9988 \pm 0.0001 $ & $ 0.0005 \pm 0.0001 $ \\
MFCN-low-pass-approx & 2 & $ 0.9972 \pm 0.0008 $ & $ 0.0011 \pm 0.0003 $ \\
MFCN-low-pass-spectral-0.5 & 2 & $ 0.9858 \pm 0.0292 $ & $ 0.0055 \pm 0.0112 $ \\
MFCN-wavelet-spectral-0.5 & 2 & $ 0.9727 \pm 0.0363 $ & $ 0.0111 \pm 0.0148 $ \\
MFCN-wavelet-spectral-1.0 & 2 & $ 0.9577 \pm 0.0373 $ & $ 0.0167 \pm 0.0147 $ \\
GIN & 2 & $ 0.0000 \pm 0.0000 $ & $ 0.3960 \pm 0.0245 $ \\
MFCN-low-pass-spectral-1.0 & 2 & $ 0.0000 \pm 0.0000 $ & $ 0.3963 \pm 0.0248 $ \\
\midrule
GAT & 3 & $ 0.9999 \pm 0.0000 $ & $ 0.0000 \pm 0.0000 $ \\
GraphSAGE & 3 & $ 0.9997 \pm 0.0004 $ & $ 0.0001 \pm 0.0001 $ \\
MFCN-low-pass-spectral-0.5 & 3 & $ 0.9991 \pm 0.0003 $ & $ 0.0003 \pm 0.0001 $ \\
MFCN-wavelet-spectral-0.5 & 3 & $ 0.9989 \pm 0.0005 $ & $ 0.0003 \pm 0.0001 $ \\
MFCN-wavelet-approx & 3 & $ 0.9989 \pm 0.0008 $ & $ 0.0004 \pm 0.0002 $ \\
MFCN-low-pass-approx & 3 & $ 0.9973 \pm 0.0016 $ & $ 0.0009 \pm 0.0004 $ \\
GCN & 3 & $ 0.9973 \pm 0.0016 $ & $ 0.0009 \pm 0.0005 $ \\
MFCN-wavelet-spectral-1.0 & 3 & $ 0.9700 \pm 0.0643 $ & $ 0.0102 \pm 0.0220 $ \\
GIN & 3 & $ 0.0000 \pm 0.0000 $ & $ 0.3238 \pm 0.0207 $ \\
MFCN-low-pass-spectral-1.0 & 3 & $ 0.0000 \pm 0.0000 $ & $ 0.3241 \pm 0.0208 $ \\
\midrule
GAT & 4 & $ 0.9994 \pm 0.0001 $ & $ 0.0002 \pm 0.0000 $ \\
MFCN-wavelet-approx & 4 & $ 0.9981 \pm 0.0024 $ & $ 0.0005 \pm 0.0007 $ \\
MFCN-wavelet-spectral-0.5 & 4 & $ 0.9967 \pm 0.0019 $ & $ 0.0009 \pm 0.0005 $ \\
MFCN-low-pass-spectral-0.5 & 4 & $ 0.9966 \pm 0.0010 $ & $ 0.0009 \pm 0.0002 $ \\
MFCN-wavelet-spectral-1.0 & 4 & $ 0.9939 \pm 0.0026 $ & $ 0.0016 \pm 0.0007 $ \\
GCN & 4 & $ 0.9930 \pm 0.0050 $ & $ 0.0019 \pm 0.0014 $ \\
MFCN-low-pass-approx & 4 & $ 0.9913 \pm 0.0102 $ & $ 0.0023 \pm 0.0027 $ \\
GraphSAGE & 4 & $ 0.9291 \pm 0.1574 $ & $ 0.0210 \pm 0.0466 $ \\
MFCN-low-pass-spectral-1.0 & 4 & $ 0.0000 \pm 0.0000 $ & $ 0.2691 \pm 0.0191 $ \\
GIN & 4 & $ 0.0000 \pm 0.0000 $ & $ 0.2706 \pm 0.0212 $ \\
\midrule
GraphSAGE & 5 & $ 0.9997 \pm 0.0001 $ & $ 0.0001 \pm 0.0000 $ \\
GAT & 5 & $ 0.9996 \pm 0.0001 $ & $ 0.0001 \pm 0.0000 $ \\
MFCN-low-pass-approx & 5 & $ 0.9987 \pm 0.0001 $ & $ 0.0003 \pm 0.0000 $ \\
MFCN-wavelet-spectral-0.5 & 5 & $ 0.9969 \pm 0.0014 $ & $ 0.0007 \pm 0.0003 $ \\
MFCN-low-pass-spectral-0.5 & 5 & $ 0.9967 \pm 0.0006 $ & $ 0.0008 \pm 0.0001 $ \\
MFCN-wavelet-spectral-1.0 & 5 & $ 0.9950 \pm 0.0017 $ & $ 0.0012 \pm 0.0004 $ \\
MFCN-wavelet-approx & 5 & $ 0.9951 \pm 0.0095 $ & $ 0.0012 \pm 0.0024 $ \\
GCN & 5 & $ 0.9909 \pm 0.0042 $ & $ 0.0021 \pm 0.0009 $ \\
MFCN-low-pass-spectral-1.0 & 5 & $ 0.0000 \pm 0.0000 $ & $ 0.2363 \pm 0.0106 $ \\
GIN & 5 & $ 0.0000 \pm 0.0000 $ & $ 0.2367 \pm 0.0100 $ \\
\midrule
GAT & 6 & $ 0.9993 \pm 0.0001 $ & $ 0.0002 \pm 0.0001 $ \\
GraphSAGE & 6 & $ 0.9992 \pm 0.0002 $ & $ 0.0002 \pm 0.0000 $ \\
MFCN-low-pass-spectral-0.5 & 6 & $ 0.9991 \pm 0.0002 $ & $ 0.0003 \pm 0.0001 $ \\
MFCN-wavelet-approx & 6 & $ 0.9989 \pm 0.0002 $ & $ 0.0003 \pm 0.0001 $ \\
MFCN-wavelet-spectral-0.5 & 6 & $ 0.9979 \pm 0.0005 $ & $ 0.0006 \pm 0.0001 $ \\
GCN & 6 & $ 0.9979 \pm 0.0002 $ & $ 0.0006 \pm 0.0001 $ \\
MFCN-wavelet-spectral-1.0 & 6 & $ 0.9972 \pm 0.0006 $ & $ 0.0008 \pm 0.0001 $ \\
MFCN-low-pass-approx & 6 & $ 0.9971 \pm 0.0011 $ & $ 0.0009 \pm 0.0003 $ \\
MFCN-low-pass-spectral-1.0 & 6 & $ 0.0000 \pm 0.0000 $ & $ 0.2953 \pm 0.0285 $ \\
GIN & 6 & $ 0.0000 \pm 0.0000 $ & $ 0.2954 \pm 0.0276 $ \\
\midrule
MFCN-low-pass-spectral-0.5 & 7 & $ 0.9972 \pm 0.0006 $ & $ 0.0006 \pm 0.0002 $ \\
GAT & 7 & $ 0.9962 \pm 0.0007 $ & $ 0.0008 \pm 0.0001 $ \\
GraphSAGE & 7 & $ 0.9948 \pm 0.0013 $ & $ 0.0012 \pm 0.0003 $ \\
MFCN-wavelet-approx & 7 & $ 0.9932 \pm 0.0065 $ & $ 0.0015 \pm 0.0014 $ \\
MFCN-low-pass-approx & 7 & $ 0.9929 \pm 0.0021 $ & $ 0.0016 \pm 0.0004 $ \\
GCN & 7 & $ 0.9926 \pm 0.0035 $ & $ 0.0016 \pm 0.0007 $ \\
MFCN-wavelet-spectral-0.5 & 7 & $ 0.9914 \pm 0.0043 $ & $ 0.0019 \pm 0.0010 $ \\
MFCN-wavelet-spectral-1.0 & 7 & $ 0.9876 \pm 0.0046 $ & $ 0.0028 \pm 0.0011 $ \\
MFCN-low-pass-spectral-1.0 & 7 & $ 0.0000 \pm 0.0000 $ & $ 0.2248 \pm 0.0152 $ \\
GIN & 7 & $ 0.0000 \pm 0.0000 $ & $ 0.2255 \pm 0.0150 $ \\
\midrule
MFCN-low-pass-spectral-0.5 & 8 & $ 0.9989 \pm 0.0002 $ & $ 0.0002 \pm 0.0000 $ \\
MFCN-wavelet-spectral-0.5 & 8 & $ 0.9974 \pm 0.0011 $ & $ 0.0005 \pm 0.0002 $ \\
MFCN-wavelet-approx & 8 & $ 0.9960 \pm 0.0023 $ & $ 0.0007 \pm 0.0004 $ \\
GraphSAGE & 8 & $ 0.9960 \pm 0.0010 $ & $ 0.0007 \pm 0.0002 $ \\
MFCN-wavelet-spectral-1.0 & 8 & $ 0.9954 \pm 0.0022 $ & $ 0.0008 \pm 0.0004 $ \\
GAT & 8 & $ 0.9949 \pm 0.0006 $ & $ 0.0009 \pm 0.0001 $ \\
GCN & 8 & $ 0.9939 \pm 0.0007 $ & $ 0.0011 \pm 0.0001 $ \\
MFCN-low-pass-approx & 8 & $ 0.9925 \pm 0.0032 $ & $ 0.0013 \pm 0.0005 $ \\
MFCN-low-pass-spectral-1.0 & 8 & $ 0.0000 \pm 0.0000 $ & $ 0.1783 \pm 0.0065 $ \\
GIN & 8 & $ 0.0000 \pm 0.0000 $ & $ 0.1784 \pm 0.0067 $ \\
\midrule
MFCN-low-pass-spectral-0.5 & 9 & $ 0.9947 \pm 0.0012 $ & $ 0.0011 \pm 0.0003 $ \\
MFCN-wavelet-approx & 9 & $ 0.9938 \pm 0.0031 $ & $ 0.0013 \pm 0.0006 $ \\
GAT & 9 & $ 0.9936 \pm 0.0024 $ & $ 0.0013 \pm 0.0004 $ \\
GraphSAGE & 9 & $ 0.9929 \pm 0.0009 $ & $ 0.0015 \pm 0.0002 $ \\
MFCN-wavelet-spectral-0.5 & 9 & $ 0.9918 \pm 0.0006 $ & $ 0.0017 \pm 0.0001 $ \\
MFCN-low-pass-approx & 9 & $ 0.9910 \pm 0.0021 $ & $ 0.0019 \pm 0.0004 $ \\
MFCN-wavelet-spectral-1.0 & 9 & $ 0.9896 \pm 0.0027 $ & $ 0.0021 \pm 0.0005 $ \\
GCN & 9 & $ 0.9886 \pm 0.0023 $ & $ 0.0024 \pm 0.0006 $ \\
MFCN-low-pass-spectral-1.0 & 9 & $ 0.0000 \pm 0.0000 $ & $ 0.2072 \pm 0.0131 $ \\
GIN & 9 & $ 0.0000 \pm 0.0000 $ & $ 0.2081 \pm 0.0145 $ \\
\midrule
MFCN-low-pass-spectral-0.5 & 10 & $ 0.9957 \pm 0.0010 $ & $ 0.0010 \pm 0.0002 $ \\
MFCN-wavelet-approx & 10 & $ 0.9940 \pm 0.0047 $ & $ 0.0014 \pm 0.0011 $ \\
GraphSAGE & 10 & $ 0.9918 \pm 0.0028 $ & $ 0.0019 \pm 0.0006 $ \\
MFCN-wavelet-spectral-0.5 & 10 & $ 0.9908 \pm 0.0022 $ & $ 0.0021 \pm 0.0004 $ \\
MFCN-low-pass-approx & 10 & $ 0.9874 \pm 0.0030 $ & $ 0.0029 \pm 0.0008 $ \\
GCN & 10 & $ 0.9859 \pm 0.0076 $ & $ 0.0032 \pm 0.0018 $ \\
GAT & 10 & $ 0.9855 \pm 0.0033 $ & $ 0.0034 \pm 0.0011 $ \\
MFCN-wavelet-spectral-1.0 & 10 & $ 0.9846 \pm 0.0063 $ & $ 0.0035 \pm 0.0015 $ \\
MFCN-low-pass-spectral-1.0 & 10 & $ 0.0000 \pm 0.0000 $ & $ 0.2306 \pm 0.0302 $ \\
GIN & 10 & $ 0.0000 \pm 0.0000 $ & $ 0.2323 \pm 0.0298 $ \\
\midrule
MFCN-low-pass-spectral-0.5 & 11 & $ 0.9966 \pm 0.0005 $ & $ 0.0005 \pm 0.0001 $ \\
MFCN-wavelet-approx & 11 & $ 0.9940 \pm 0.0010 $ & $ 0.0009 \pm 0.0002 $ \\
MFCN-wavelet-spectral-0.5 & 11 & $ 0.9939 \pm 0.0006 $ & $ 0.0009 \pm 0.0001 $ \\
MFCN-wavelet-spectral-1.0 & 11 & $ 0.9903 \pm 0.0042 $ & $ 0.0015 \pm 0.0006 $ \\
GAT & 11 & $ 0.9903 \pm 0.0020 $ & $ 0.0015 \pm 0.0004 $ \\
GraphSAGE & 11 & $ 0.9894 \pm 0.0015 $ & $ 0.0017 \pm 0.0002 $ \\
MFCN-low-pass-approx & 11 & $ 0.9837 \pm 0.0037 $ & $ 0.0025 \pm 0.0005 $ \\
GCN & 11 & $ 0.9838 \pm 0.0026 $ & $ 0.0025 \pm 0.0004 $ \\
MFCN-low-pass-spectral-1.0 & 11 & $ 0.0000 \pm 0.0000 $ & $ 0.1571 \pm 0.0075 $ \\
GIN & 11 & $ 0.0000 \pm 0.0000 $ & $ 0.1577 \pm 0.0082 $ \\
\midrule
MFCN-wavelet-approx & 12 & $ 0.9734 \pm 0.0041 $ & $ 0.0014 \pm 0.0002 $ \\
MFCN-low-pass-spectral-0.5 & 12 & $ 0.9726 \pm 0.0070 $ & $ 0.0014 \pm 0.0004 $ \\
MFCN-wavelet-spectral-0.5 & 12 & $ 0.9495 \pm 0.0068 $ & $ 0.0027 \pm 0.0004 $ \\
MFCN-wavelet-spectral-1.0 & 12 & $ 0.9292 \pm 0.0173 $ & $ 0.0037 \pm 0.0010 $ \\
GraphSAGE & 12 & $ 0.9099 \pm 0.0186 $ & $ 0.0047 \pm 0.0009 $ \\
MFCN-low-pass-approx & 12 & $ 0.8970 \pm 0.0202 $ & $ 0.0054 \pm 0.0011 $ \\
GCN & 12 & $ 0.3481 \pm 0.4999 $ & $ 0.0349 \pm 0.0270 $ \\
GAT & 12 & $ 0.1625 \pm 0.4006 $ & $ 0.0442 \pm 0.0213 $ \\
MFCN-low-pass-spectral-1.0 & 12 & $ 0.0000 \pm 0.0000 $ & $ 0.0529 \pm 0.0020 $ \\
GIN & 12 & $ 0.0000 \pm 0.0000 $ & $ 0.0529 \pm 0.0019 $ \\
\midrule
MFCN-low-pass-spectral-0.5 & 13 & $ 0.9800 \pm 0.0021 $ & $ 0.0023 \pm 0.0002 $ \\
MFCN-wavelet-spectral-0.5 & 13 & $ 0.9756 \pm 0.0065 $ & $ 0.0028 \pm 0.0007 $ \\
MFCN-wavelet-approx & 13 & $ 0.9649 \pm 0.0062 $ & $ 0.0041 \pm 0.0008 $ \\
MFCN-wavelet-spectral-1.0 & 13 & $ 0.9638 \pm 0.0113 $ & $ 0.0042 \pm 0.0013 $ \\
MFCN-low-pass-approx & 13 & $ 0.9107 \pm 0.0336 $ & $ 0.0105 \pm 0.0045 $ \\
GraphSAGE & 13 & $ 0.8933 \pm 0.0172 $ & $ 0.0123 \pm 0.0017 $ \\
MFCN-low-pass-spectral-1.0 & 13 & $ 0.0000 \pm 0.0000 $ & $ 0.1162 \pm 0.0107 $ \\
GCN & 13 & $ 0.0000 \pm 0.0000 $ & $ 0.1165 \pm 0.0110 $ \\
GIN & 13 & $ 0.0000 \pm 0.0000 $ & $ 0.1166 \pm 0.0108 $ \\
GAT & 13 & $ 0.0000 \pm 0.0000 $ & $ 0.1172 \pm 0.0104 $ \\
\midrule
MFCN-low-pass-spectral-0.5 & 14 & $ 0.9875 \pm 0.0025 $ & $ 0.0016 \pm 0.0004 $ \\
MFCN-wavelet-spectral-0.5 & 14 & $ 0.9580 \pm 0.0253 $ & $ 0.0051 \pm 0.0028 $ \\
MFCN-wavelet-spectral-1.0 & 14 & $ 0.9583 \pm 0.0057 $ & $ 0.0053 \pm 0.0010 $ \\
MFCN-wavelet-approx & 14 & $ 0.9577 \pm 0.0052 $ & $ 0.0053 \pm 0.0009 $ \\
GraphSAGE & 14 & $ 0.8728 \pm 0.0246 $ & $ 0.0159 \pm 0.0025 $ \\
MFCN-low-pass-approx & 14 & $ 0.7761 \pm 0.1323 $ & $ 0.0278 \pm 0.0159 $ \\
MFCN-low-pass-spectral-1.0 & 14 & $ 0.0000 \pm 0.0000 $ & $ 0.1262 \pm 0.0114 $ \\
GAT & 14 & $ 0.0000 \pm 0.0000 $ & $ 0.1263 \pm 0.0109 $ \\
GIN & 14 & $ 0.0000 \pm 0.0000 $ & $ 0.1269 \pm 0.0118 $ \\
GCN & 14 & $ 0.0000 \pm 0.0000 $ & $ 0.1274 \pm 0.0099 $ \\
\midrule
MFCN-low-pass-spectral-0.5 & 15 & $ 0.9820 \pm 0.0038 $ & $ 0.0022 \pm 0.0004 $ \\
MFCN-wavelet-spectral-1.0 & 15 & $ 0.9659 \pm 0.0129 $ & $ 0.0043 \pm 0.0015 $ \\
MFCN-wavelet-spectral-0.5 & 15 & $ 0.9620 \pm 0.0151 $ & $ 0.0048 \pm 0.0021 $ \\
MFCN-wavelet-approx & 15 & $ 0.9324 \pm 0.0236 $ & $ 0.0085 \pm 0.0031 $ \\
MFCN-low-pass-approx & 15 & $ 0.8407 \pm 0.0648 $ & $ 0.0196 \pm 0.0069 $ \\
GraphSAGE & 15 & $ 0.8444 \pm 0.0240 $ & $ 0.0196 \pm 0.0037 $ \\
GAT & 15 & $ 0.1213 \pm 0.2930 $ & $ 0.1108 \pm 0.0385 $ \\
MFCN-low-pass-spectral-1.0 & 15 & $ 0.0000 \pm 0.0000 $ & $ 0.1266 \pm 0.0091 $ \\
GIN & 15 & $ 0.0000 \pm 0.0000 $ & $ 0.1267 \pm 0.0087 $ \\
GCN & 15 & $ 0.0000 \pm 0.0000 $ & $ 0.1268 \pm 0.0086 $ \\
\midrule
MFCN-low-pass-spectral-0.5 & 16 & $ 0.9195 \pm 0.0268 $ & $ 0.0077 \pm 0.0022 $ \\
MFCN-wavelet-approx & 16 & $ 0.8587 \pm 0.0899 $ & $ 0.0132 \pm 0.0075 $ \\
MFCN-wavelet-spectral-0.5 & 16 & $ 0.6469 \pm 0.3695 $ & $ 0.0328 \pm 0.0327 $ \\
MFCN-wavelet-spectral-1.0 & 16 & $ 0.4320 \pm 0.3872 $ & $ 0.0535 \pm 0.0330 $ \\
MFCN-low-pass-approx & 16 & $ 0.2193 \pm 0.3308 $ & $ 0.0772 \pm 0.0365 $ \\
GIN & 16 & $ 0.0000 \pm 0.0000 $ & $ 0.0980 \pm 0.0108 $ \\
MFCN-low-pass-spectral-1.0 & 16 & $ 0.0000 \pm 0.0000 $ & $ 0.0983 \pm 0.0104 $ \\
GAT & 16 & $ 0.0000 \pm 0.0000 $ & $ 0.0987 \pm 0.0103 $ \\
GCN & 16 & $ 0.0000 \pm 0.0000 $ & $ 0.0996 \pm 0.0107 $ \\
GraphSAGE & 16 & $ 0.0000 \pm 0.0000 $ & $ 0.1202 \pm 0.0279 $ \\
\midrule
MFCN-low-pass-spectral-0.5 & 17 & $ 0.9769 \pm 0.0070 $ & $ 0.0012 \pm 0.0004 $ \\
MFCN-wavelet-spectral-0.5 & 17 & $ 0.8994 \pm 0.0558 $ & $ 0.0053 \pm 0.0026 $ \\
MFCN-wavelet-approx & 17 & $ 0.8966 \pm 0.0191 $ & $ 0.0056 \pm 0.0009 $ \\
MFCN-wavelet-spectral-1.0 & 17 & $ 0.6913 \pm 0.3751 $ & $ 0.0159 \pm 0.0182 $ \\
GIN & 17 & $ 0.0000 \pm 0.0000 $ & $ 0.0543 \pm 0.0055 $ \\
GraphSAGE & 17 & $ 0.0000 \pm 0.0000 $ & $ 0.0546 \pm 0.0059 $ \\
GAT & 17 & $ 0.0000 \pm 0.0000 $ & $ 0.0546 \pm 0.0060 $ \\
MFCN-low-pass-approx & 17 & $ 0.0000 \pm 0.0000 $ & $ 0.0548 \pm 0.0063 $ \\
GCN & 17 & $ 0.0000 \pm 0.0000 $ & $ 0.0548 \pm 0.0063 $ \\
MFCN-low-pass-spectral-1.0 & 17 & $ 0.0000 \pm 0.0000 $ & $ 0.0549 \pm 0.0065 $ \\
\midrule
MFCN-low-pass-spectral-0.5 & 18 & $ 0.9660 \pm 0.0092 $ & $ 0.0033 \pm 0.0007 $ \\
MFCN-wavelet-spectral-1.0 & 18 & $ 0.8598 \pm 0.0394 $ & $ 0.0135 \pm 0.0029 $ \\
MFCN-wavelet-spectral-0.5 & 18 & $ 0.8428 \pm 0.1018 $ & $ 0.0152 \pm 0.0098 $ \\
MFCN-wavelet-approx & 18 & $ 0.8113 \pm 0.0614 $ & $ 0.0181 \pm 0.0040 $ \\
MFCN-low-pass-spectral-1.0 & 18 & $ 0.0000 \pm 0.0000 $ & $ 0.0989 \pm 0.0142 $ \\
MFCN-low-pass-approx & 18 & $ 0.0000 \pm 0.0000 $ & $ 0.0989 \pm 0.0142 $ \\
GIN & 18 & $ 0.0000 \pm 0.0000 $ & $ 0.0993 \pm 0.0141 $ \\
GAT & 18 & $ 0.0000 \pm 0.0000 $ & $ 0.1005 \pm 0.0146 $ \\
GCN & 18 & $ 0.0000 \pm 0.0000 $ & $ 0.1019 \pm 0.0154 $ \\
GraphSAGE & 18 & $ 0.0000 \pm 0.0000 $ & $ 0.1045 \pm 0.0492 $ \\
\midrule
MFCN-low-pass-spectral-0.5 & 19 & $ 0.9429 \pm 0.0270 $ & $ 0.0050 \pm 0.0019 $ \\
MFCN-wavelet-approx & 19 & $ 0.7251 \pm 0.0563 $ & $ 0.0248 \pm 0.0058 $ \\
MFCN-wavelet-spectral-0.5 & 19 & $ 0.6070 \pm 0.4188 $ & $ 0.0333 \pm 0.0333 $ \\
MFCN-wavelet-spectral-1.0 & 19 & $ 0.4149 \pm 0.4227 $ & $ 0.0503 \pm 0.0340 $ \\
MFCN-low-pass-spectral-1.0 & 19 & $ 0.0000 \pm 0.0000 $ & $ 0.0905 \pm 0.0083 $ \\
MFCN-low-pass-approx & 19 & $ 0.0000 \pm 0.0000 $ & $ 0.0906 \pm 0.0083 $ \\
GIN & 19 & $ 0.0000 \pm 0.0000 $ & $ 0.0906 \pm 0.0083 $ \\
GraphSAGE & 19 & $ 0.0000 \pm 0.0000 $ & $ 0.0910 \pm 0.0076 $ \\
GAT & 19 & $ 0.0000 \pm 0.0000 $ & $ 0.0919 \pm 0.0088 $ \\
GCN & 19 & $ 0.0000 \pm 0.0000 $ & $ 0.0930 \pm 0.0081 $ \\
\midrule
MFCN-low-pass-spectral-0.5 & 20 & $ 0.9538 \pm 0.0242 $ & $ 0.0029 \pm 0.0014 $ \\
MFCN-wavelet-spectral-0.5 & 20 & $ 0.9176 \pm 0.0469 $ & $ 0.0051 \pm 0.0026 $ \\
MFCN-wavelet-spectral-1.0 & 20 & $ 0.9123 \pm 0.0283 $ & $ 0.0056 \pm 0.0017 $ \\
MFCN-wavelet-approx & 20 & $ 0.8322 \pm 0.0273 $ & $ 0.0107 \pm 0.0010 $ \\
GraphSAGE & 20 & $ 0.0547 \pm 0.1320 $ & $ 0.0608 \pm 0.0114 $ \\
GIN & 20 & $ 0.0000 \pm 0.0000 $ & $ 0.0643 \pm 0.0067 $ \\
MFCN-low-pass-approx & 20 & $ 0.0000 \pm 0.0000 $ & $ 0.0645 \pm 0.0068 $ \\
MFCN-low-pass-spectral-1.0 & 20 & $ 0.0000 \pm 0.0000 $ & $ 0.0646 \pm 0.0069 $ \\
GAT & 20 & $ 0.0000 \pm 0.0000 $ & $ 0.0653 \pm 0.0082 $ \\
GCN & 20 & $ 0.0000 \pm 0.0000 $ & $ 0.0656 \pm 0.0089 $ \\
\midrule
MFCN-low-pass-spectral-0.5 & 21 & $ 0.9586 \pm 0.0099 $ & $ 0.0037 \pm 0.0008 $ \\
MFCN-wavelet-spectral-0.5 & 21 & $ 0.9307 \pm 0.0165 $ & $ 0.0062 \pm 0.0019 $ \\
MFCN-wavelet-spectral-1.0 & 21 & $ 0.8283 \pm 0.1164 $ & $ 0.0147 \pm 0.0079 $ \\
MFCN-wavelet-approx & 21 & $ 0.8079 \pm 0.0672 $ & $ 0.0174 \pm 0.0072 $ \\
MFCN-low-pass-approx & 21 & $ 0.1057 \pm 0.2451 $ & $ 0.0819 \pm 0.0275 $ \\
MFCN-low-pass-spectral-1.0 & 21 & $ 0.0000 \pm 0.0000 $ & $ 0.0898 \pm 0.0107 $ \\
GAT & 21 & $ 0.0000 \pm 0.0000 $ & $ 0.0899 \pm 0.0107 $ \\
GIN & 21 & $ 0.0000 \pm 0.0000 $ & $ 0.0900 \pm 0.0108 $ \\
GCN & 21 & $ 0.0000 \pm 0.0000 $ & $ 0.0900 \pm 0.0108 $ \\
GraphSAGE & 21 & $ 0.0000 \pm 0.0000 $ & $ 0.0950 \pm 0.0405 $ \\
\botrule
\end{longtable}

\section{Further experimental details}\label{appendix:experimental_details}

\subsection{General} 
\noindent \textit{MFCN-low-pass models}
\begin{itemize}
    \item Single low-pass filter, $w(\lambda) = e^{-t\lambda}$, applied in the generalized Fourier domain of a $k$-NN graph's Laplacian's first 20 nontrivial eigenvectors (spectral) or $w(\mathbf{P}_n) = \mathbf{P}_n$ (approx). 
    \item Omit the Combine-Filters step.
    \item Two layers/cycles, the first with 32 output channels and the second with 16.
    \item ReLU activation in step (iv)
\end{itemize}

\noindent \textit{Baseline models}
\begin{itemize}
    \item Two layers/cycles, each with 64 output channels, for each baseline module.
    \item Used the GNN module implementations of the python package PyTorch Geometric, named `GCN', `GAT', `GraphSAGE', and `GIN.'
\end{itemize}

\noindent \textit{Training rules}
\begin{itemize}
    \item For each cross-validation fold within each experiment, models are trained up until a minimum `burn-in' number of epochs, or until loss for the hold-out (validation) set fails to improve (by some relative amount) for a fixed number of `patience' epochs. Then, final model weights are those achieved in the epoch of least hold-out set loss, which may fall before the burn-in number of epochs was satisfied. Performance metrics are then calculated using these final models weights on the hold-out set.
    \item The specific burn-in and patience parameters for each experiment are detailed in their respective sections.
\end{itemize}

\noindent \textit{Random seeds}
\begin{itemize}
    \item Applicable random seeds for each experiment can be found in the `args.py' files in the code base, posted on GitHub.
\end{itemize}

\noindent \textit{Hardware}
\begin{itemize}
    \item Models were fit on a compute node that allocated 16 cores of an Intel Xeon Platinum 8562Y 2.8GHz CPU and one Nvidia L40 GPU.
\end{itemize}

\subsection{Ellipsoid node regression experiments}\label{appendix:exp_details_ellipsoids}

\noindent \textit{MFCN-wavelet models}
\begin{itemize}
    \item Dyadic wavelet filter banks, up to $2^J = 32$ ($J = 5$) and including $2^5$ as the low-pass filter; generalized Fourier domain constructed with the first 20 nontrivial eigenvectors for MFCN-spectral.
    \item Two total filter-combine cycles, with eight combinations out in the first cycle and four combinations out in the second cycle, for both the cross-channel (ii) and cross-filter (iii) steps.
    \item ReLU activation in step (iv)
\end{itemize}

\noindent \textit{Final signal channel pooling within nodes}
\begin{itemize}
    \item For each model, after its two convolutional cycles, we used a single final learnable linear layer to transform nodes' hidden layer embeddings into one scalar output prediction for each node, i.e $\widehat{\mathbf{y}}(i) \in \mathbb{R}^{1}$.
\end{itemize}

\noindent \textit{Data}
\begin{itemize}
    \item To numerically assist model learning, all regression target vectors $y$ were mapped onto the interval $[-1, 1]$ using min-max scaling (i.e. $f(y) = (y - \mathrm{min}(y)) / (\mathrm{max}(y) - \mathrm{min}(y))$  followed by the linear transformation $2f(y) - 1$.
\end{itemize}

\noindent \textit{Training regime}
\begin{itemize}
    \item AdamW optimizer, with 0.9 and 0.999 for the first- and second-moment gradient decay rates, and an $\ell^2$-regularization coefficient of $1 \times 10^{-2}$.
    \item Learning rate of 0.01.
    \item No minibatching (each whole ellipsoid graph fit into memory).
    \item Mean squared error loss.
    \item Stopping rule of no (arbitrarily small) improvement in validation loss for 50 epochs, after a burn-in of 100 epochs.
    \item  During training, as is standard practice for node-level tasks, each model's forward pass used all signal channels of all nodes; but during the loss calculation and parameter update steps, the target values of the nodes in the validation set were masked from the optimizer.
\end{itemize}

\subsection{Melanoma patient manifold classification experiments}\label{appendix:exp_details_melanoma}

\noindent \textit{MFCN-wavelet models}
\begin{itemize}
    \item Spectral and approx-dyadic: Dyadic wavelet filter banks, up to $2^J = 32$ (i.e., $J = 5$) and including $2^5$ as the low-pass filter.
    \item Spectral: generalized Fourier domain constructed with the first 20 nontrivial eigenvectors.
    \item Approx-Infogain: `Infogain' wavelet scales for the initial filter bank using a maximum diffusion step $t$ of 32, and information partitioning quantiles of (0.125, 0.25, 0.375, 0.5, 0.625, 0.75, 0.875); then dyadic scales for the second cycle's filter bank (to reduce computational complexity; less necessary in the second cycle with its learned recombinations, as opposed to initial feature extraction directly from the data).
    \item Two total filter-combine cycles, with 16 combinations out in the first cycle and eight combinations out in the second cycle for cross-filter (iii) steps; no cross-channel combination step (ii).
    \item ReLU activation in step (iv).
\end{itemize}

\noindent \textit{Fully-connected network classifier head}
\begin{itemize}
    \item For each graph, each final output channel from MFCN or GNN models was max pooled, and then these channel maximums were concatenated into a vector (or more precisely, a tensor of size `number of graphs' $\times$ `number of output channels' when batch training), which served as the embedded features vector fed into the fully-connected network.
    \item Five linear layers, with sizes (128, 64, 32, 16).
    \item Batch normalization in between each hidden layer; no dropout.
    \item ReLU activations in between hidden layers.
\end{itemize}

\noindent \textit{Data}
\begin{itemize}
    \item The raw protein expression counts for the 29 proteins of interest were $\mathrm{log}_{2}$-scaled and then $\ell^{1}$-normalized in each T-cell sample, such that the sum of all $\mathrm{log}_{2}$ expression values in the sample summed to 1. However, raw protein expression counts of zero were kept at zero.
\end{itemize}

\noindent \textit{Training regime}
\begin{itemize}
    \item AdamW optimizer, with 0.9 and 0.999 for the first- and second-moment gradient decay rates, and an $\ell^2$-regularization coefficient of $1 \times 10^{-5}$.
    \item Learning rate of 0.005.
    \item Minibatch size of 8 graphs, to mitigate overfitting.
    \item Class-balanced binary cross-entropy loss (with logits); loss re-balancing weights were calculated from the training set of each fold to ensure equal contribution to loss from each class (since the melanoma dataset was unbalanced towards the negative `nonresponder' class, approximately 58\% to 42\%).
    \item Stopping rule of no (arbitrarily small) improvement in validation loss for 50 epochs, after a burn-in of 250 epochs.
\end{itemize}

\end{appendices}

\end{document}